\documentclass[a4paper]{article}

\usepackage[letterpaper, left=1.2truein, right=1.2truein, top = 1.2truein, bottom = 1.2truein]{geometry}

\usepackage[english]{babel}
\usepackage[utf8x]{inputenc}
\usepackage[T1]{fontenc}

\usepackage[dvipsnames]{xcolor}
\usepackage{graphicx}
\usepackage{subcaption}
\usepackage{longtable} 
\usepackage{multirow}
\usepackage{listings}
\usepackage{makecell}
\usepackage{array}
\usepackage{float}
\usepackage{dsfont}
\usepackage{rotating}
\usepackage{booktabs}
\usepackage{enumitem}
\usepackage{tikz}
\usepackage{pgf}
\usepackage{xcolor}

\usepackage{amsmath}
\usepackage{amssymb}
\usepackage{amsthm}
\usepackage{bm}
\usepackage[linesnumbered,ruled]{algorithm2e}
\usepackage{algorithmicx}
\usepackage{mathtools}
\usepackage{mathrsfs}

\usepackage{tcolorbox}
\usepackage{wrapfig}
\usepackage{diagbox}

\usepackage{parskip}
\allowdisplaybreaks

\usepackage{natbib}
\usepackage[
  colorlinks,
  citecolor=blue,
  linkcolor=red,
  anchorcolor=red,
  urlcolor=blue
]{hyperref}
\mathtoolsset{showonlyrefs}
\usepackage{authblk}
\usepackage{todonotes}
\theoremstyle{plain}

\newtheorem{theorem}{Theorem}
\newtheorem{lemma}{Lemma}

\newtheorem{corollary}{Corollary}

\newtheorem{assumption}{Assumption}

\newcommand{\RR}{\mathbb{R}}
\newcommand{\EE}{\mathbb{E}}
\newcommand{\PP}{\mathbb{P}}
\newcommand{\ind}{\mathds{1}}

\newcommand{\sign}{\textnormal{sign}}

\newcommand{\argmax}[1]{\underset{#1}{\arg\!\max}}
\newcommand{\argmin}[1]{\underset{#1}{\arg\!\min}}

\usepackage{xspace}

\newcommand{\stepa}[1]{\overset{\rm (a)}{#1}}
\newcommand{\stepb}[1]{\overset{\rm (b)}{#1}}

\newcommand{\ceil}[1]{{\left\lceil {#1} \right \rceil}}
\newcommand{\given}{{\,|\,}}

\def\@#1\@{\begin{align}#1\end{align}}
\def\$#1\${\begin{align*}#1\end{align*}}

\definecolor{myblue}{rgb}{.8, .8, 1}
\definecolor{mathblue}{rgb}{0.2472, 0.24, 0.6} 
\definecolor{mathred}{rgb}{0.6, 0.24, 0.442893}
\definecolor{mathyellow}{rgb}{0.6, 0.547014, 0.24}

\newcommand{\sfA}{{\mathsf{A}}}

\newcommand{\sfE}{{\mathsf{E}}}

\newcommand{\sfI}{{\mathsf{I}}}

\newcommand{\sfM}{{\mathsf{M}}}

\newcommand{\sfY}{{\mathsf{Y}}}

\newcommand{\cA}{{\mathcal{A}}}

\newcommand{\cE}{{\mathcal{E}}}

\newcommand{\cG}{{\mathcal{G}}}

\newcommand{\cN}{{\mathcal{N}}}

\newcommand{\one}{\mathbf{1}}

\newcommand{\supp}{\mathsf{supp}}
\newcommand{\regret}{\mathsf{regret}}
\newcommand{\reward}{\mathsf{reward}}

\newcommand{\netc}{\mathsf{netc}}
\newcommand{\nsefs}{\mathsf{nsefs}}
\newcommand{\nse}{\mathsf{nse}}
\renewcommand{\tilde}{\widetilde}
\long\def\comment#1{}

\title{Learning to target with network interference}
\author[1]{Xiaomeng Wang}
\affil[1]{Department of Statistics and Data Science, University of Pennsylvania}
\author[2]{Hamsa Bastani}
\affil[2]{Operations, Information and Decisions Department, University of Pennsylvania}
\author[3]{Osbert Bastani}
\affil[3]{Department of Computer and Information Science, University of Pennsylvania}
\author[1]{Zhimei Ren}
\date{\today}

\begin{document}
\maketitle
\begin{abstract}
This paper studies adaptive targeting
under network interference in a bandit setting, where treatments applied to one individual may affect others through 
spillover effects. We consider a linear model in a sparse regime, 
where each individual's outcome can be affected by at most a few others.
We first establish a regret lower bound showing that ignoring the network structure and reducing the problem to a 
standard linear bandit inevitably leads to inefficient learning, particularly in large populations. 
To understand how structural information can be leveraged, we
analyze regimes with varying levels of knowledge of the interference structure: 
(1) full support knowledge, (2) knowledge of the column support sizes, and (3) no prior knowledge. 
For each regime, we establish regret lower bounds characterizing the fundamental limits of learning,
and develop algorithms that achieve near-optimal regret. 
Together, our results provide a unified view of how knowledge of the interference structure 
governs the efficiency of online learning under interference, and offer practical adaptive targeting algorithms in each setting.
Numerical experiments on synthetic and real-world data demonstrate the practical benefits of our algorithms.
\end{abstract}

\section{Introduction}

Optimal targeting policies seek to target a specific set of individuals in a population to maximize an overall reward. For instance, suppose we are rolling out an advertisement through push notifications. Ideally, we want to target users whose \textit{treatment effects}---i.e., their expected increase in conversion from being shown an ad---are positive. Advances in machine learning have enabled estimation of heterogeneous treatment effects~\citep{athey2016recursive} that inform targeting policies in a number of domains like healthcare~\citep{bastani2021efficient}, recommendation systems~\citep{li2010contextual}, and microfinance~\citep{banerjee2013diffusion}.

However, a key challenge to estimating treatment effects is \textit{network interference}~\citep{ugander2013graph,eckles2017design,li2022random,li2022interference,johari2022experimental,hu2022average,shirani2023causal, agarwal_multi-armed_2024}, where treated individuals may affect the outcomes of control units.
For example, a user who sees the ad might recommend the product to their friends, increasing their likelihood 
to purchase the product (even though they did not see the ad themselves). 
In other words, the ad affects not only the treated individual but also any control individuals connected to the treated individuals. 
Consequently, the naive estimation strategy that simply compares treatment and control outcomes would underestimate the effectiveness of the ad by overestimating the outcomes of control individuals who benefited indirectly from the ad. 
In general, these spillovers can be positive (as in the advertising example) or negative---e.g., 
sharing access to a digital subscription (e.g., Netflix) might reduce others' incentives to purchase their own subscription.
These considerations make it challenging to learn optimal targeting policies under network interference.

In this paper, we study optimal targeting under network interference, with the goal  
of learning the relevant interference structure and optimizing interventions in an online setting.
Over a time horizon $T$, a decision-maker sequentially chooses the actions for all individuals in the population,
observes the resulting rewards for each individual, and seeks to minimize cumulative regret relative to the best 
fixed action in hindsight. Unlike the standard bandit setting, where the reward for each individual depends only on the action taken on that individual,
network interference allows one individual's action to affect  
the outcome of others through the heterogeneous spillover effects across the population. 
This leads to a high-dimensional learning problem, with a large number of parameters representing both direct and spillover effects, 
especially when the population size is large. 

\subsection{Our contributions}
To tackle the problem, 
we consider the {\em linear treatment effect model}~\citep{toulis2013estimation,eckles2017design,sussman2017elements,hu2022average,yu2022estimating}, 
where the reward for each individual is a linear function of the actions taken on themselves and their neighbors in an interference network:
for a population of size $d$, the mean reward for individual $i$ is given by
\$ 
\mathsf{reward}_i = \sum_{j=1}^d X^*_{ij} a_j,
\$
where $a_j$ is the action taken on individual $j$, 
and $X^*_{ij}$ is the effect of the action assigned to
individual $j$ on the reward of individual $i$. 
This is an expressive model that has been widely used in the literature on causal inference under interference;
more discussion of this model is provided in Section~\ref{sec2}.

In addition, we assume that the interference matrix $X^*$ is {\em row-sparse}, meaning that each individual's reward depends on at most $s$ other individuals, for some $s\ll d$. 
The intuition is that---even if population size is large---each individual 
can only reasonably be influenced by a small number of others for a given intervention. 
(Note that this does not rule out ``celebrity'' individuals who can influence many individuals, 
as long as they themselves are only influenced by a small number of other individuals.) In our ad targeting example, this implies that each individual's decision to purchase a product is only influenced by a few unknown ``family and friends'' or trusted sources.

Under this model, our goal is to design an adaptive targeting algorithm 
that maximizes the cumulative reward across {\em all individuals} over the time horizon $T$. To this end, a straightforward strategy is to forgo learning the network structure
and instead directly optimize the \textit{aggregated} treatment effect by using the sum of rewards from all individuals, 
rather than the vector of individual rewards.
This strategy reduces to a single standard stochastic linear bandit problem 
(see details in Section~\ref{sec3}), which can be solved via 
off-the-shelf algorithms~\citep[e.g., LinUCB, see][]{abbasi2011improved}. 
The reduced problem, however, is not necessarily sparse, 
even if the original network is sparse, 
since the aggregated reward for each individual can depend on all actions through spillovers.
In Section~\ref{sec3}, we prove an instance-dependent lower bound (Theorem~\ref{thm:instance-baselineLB}), which implies 
that any policy operating on aggregated rewards must incur regret at least 
$\Omega(d^{3/2}T^{1/2} \wedge dT)$ in the worst case (Corollary~\ref{cor:minimax-baselineLB}).
The lower bound suggests that this reduction strategy is fundamentally inefficient, 
with regret that grows super-linearly in $d$ or linearly in $T$. 

Motivated by this impossibility result, we seek to understand {\em whether} and {\em how}
network structural information can be leveraged  
to achieve better performance.
In particular, we consider three settings with increasing levels of difficulty, 
depending on the amount of prior information about the network structure: 
(i) {\em full support knowledge}, where the network support is known but effect magnitudes are unknown; 
(ii) {\em partial structural knowledge}, where aggregate information---column support sizes---is known; 
and (iii) {\em no prior support information beyond row-sparsity}. 
In each setting, we establish the minimax regret lower bound that characterizes the intrinsic difficulty of learning, 
and propose efficient bandit algorithms that (nearly) match the lower bound. Specifically, our main results are:

\begin{itemize}
\item[(1)] {\em Instance-dependent lower bounds.}
We establish a series of instance-dependent lower bounds that characterize the intrinsic difficulty of learning under network interference with different levels of structural information.
For the aggregated reward setting, we prove that given a matrix support specification,
any policy operating on aggregated rewards must incur regret at least 
$\Omega(\sqrt{T} \sum_{j=1}^d
\min \{ \sqrt{d}\cdot \ind\{\rho_j > 0\}, 2\sqrt{T} \rho_j/s\})$,
where $\rho_j$ is the column support size of the interference network (Theorem~\ref{thm:instance-baselineLB}),
and the lower bound is $\Omega(d^{3/2}T^{1/2} \wedge dT)$ in the worst case  
(Corollary~\ref{cor:minimax-baselineLB}).
When individual rewards are used, we prove that 
the instance-dependent regret lower bound is
$\Omega (\sqrt{T}\sum_{j=1}^d \sqrt{\rho_j})$ (Theorem~\ref{thm:LB}),
and becomes $\Omega(d\sqrt{sT})$ (Corollary~\ref{cor:worst-case-LB}) in the worst case.

\item[(2)] {\em Efficient algorithms under different levels of network information.}
We propose efficient bandit algorithms for learning under network interference with 
different levels of structural information, and establish regret upper bounds that (nearly) match the lower bounds in each setting.
When the network support is fully known, we propose the 
{\em network successive elimination with full support knowledge (NSE-FS)} algorithm, 
which achieves regret $\tilde{{O}}(\sqrt{T}\sum_{j=1}^d \sqrt{\rho_j})$ 
(Theorem~\ref{thm:full-support}), matching the corresponding lower bound up to logarithmic factors. 
When only column support sizes are known, we propose the {\em network successive elimination (NSE)} 
algorithm, which achieves regret $\tilde{O}(\sqrt{T}\sum_{j=1}^d \rho_j)$ (Theorem~\ref{thm:se_upper_bound}),
which is near-optimal up to a $\sqrt{\rho_j}$ factor and logarithmic terms.
Finally, when the network is completely unknown, we propose the {\em network explore-then-commit (NETC)} 
algorithm with regret $\tilde{O}(d(sT)^{2/3})$ (Theorem~\ref{thm:dpUB}),
which is optimal in the population size $d$.
\item[(3)] {\em Empirical results.} We conduct extensive simulations to validate 
and demonstrate the efficiency and robustness of our algorithms, including 
a semi-synthetic experiment on a real-world social network dataset. 
Our results show that our algorithms are highly scalable and can effectively learn optimal targeting policies under network interference, 
outperforming baseline methods that ignore network structure.
\end{itemize}

Together, these results characterize how regret scales with the population size $d$, sparsity level $s$, and time horizon $T$ under different amounts of available network information. We summarize our results in Table~\ref{t1}.

\begin{table}[h]
\begin{center}
\begin{tabular}{l|c|c|l} 
\toprule
\rule{0pt}{2.5ex}   
Algorithm 
& Upper bound 
& Lower bound 
& Network information (beyond row sparsity)\\ 
\midrule

\rule{0pt}{2.5ex}
Baseline 
& --- 
& $\Omega(d^{3/2}T^{1/2} \wedge dT)$
& None (aggregated reward) \\[0.8ex]

\rule{0pt}{2.5ex}
NSE-FS (Alg.~\ref{alg:full-support})
& $\tilde{O}(\sqrt{T}\sum_{j=1}^d \sqrt{\rho_j})$
& $\Omega(\sqrt{T}\sum_{j=1}^d \sqrt{\rho_j})$
& Full support known \\[0.8ex]

\rule{0pt}{2.5ex}
NSE (Alg.~\ref{alg:col-support})
& $\tilde{O}(\sqrt{T}\sum_{j=1}^d \rho_j )$
& $\Omega(\sqrt{T}\sum_{j=1}^d \sqrt{\rho_j})$
& Column support sizes known \\[0.8ex]

\rule{0pt}{2.5ex}
NETC (Alg.~\ref{alg:netc})
& $\tilde{O}(d (sT)^{2/3})$
& $\Omega(d\sqrt{sT})$
& No support information \\

\bottomrule
\end{tabular}
\end{center}
\caption{Summary of regret bounds under different levels of network information. 
Here $\rho_j$ denotes the column support size and $\sum_{j=1}^d \rho_j \le ds$ under row sparsity.}
\label{t1}
\end{table}

\paragraph{Prior work on online learning under network interference.}
\citet{agarwal_multi-armed_2024} study a similar problem 
of multi-armed bandits under sparse network interference. They consider a more general 
model where the reward is an arbitrary function of the actions taken on all individuals, 
and a Boolean Fourier representation is used to model the reward function. 
Under their model, the authors provide algorithms for the fully known and unknown network settings. 
In the fully known network setting, two algorithms are proposed with regret---using our notation---of the order  
$\tilde{O}\bigl(d(T|\cA|^s)^{2/3}\bigr)$ and $\tilde{O}(d^{3/2}\sqrt{|\cA|^sT})$, respectively, 
where $|\cA|$ is the number of possible actions for each individual. 
In the unknown network setting, their proposed algorithm achieves regret of the order 
$\tilde{O}\bigl(d^{4/3}(T|\cA|^s)^{2/3}\bigr)$.
In both cases, the regret scales exponentially in the sparsity level $s$, 
and either depends on $T$ with a $2/3$ exponent 
and/or scales superlinearly in $d$; it is also not clear whether the dependence 
is optimal under their model. In practice, the computation and memory requirements of these algorithms 
can be prohibitive when $s$ is moderately large.
In our work, we provide a more fine-grained analysis of the regret 
with different levels of network information with the linear 
treatment effect model.
In each setting, we establish the minimax regret lower bound that characterizes the intrinsic difficulty of learning,
and propose efficient bandit algorithms that (nearly) match the lower bound.
In addition, our algorithms are computationally efficient and scalable to large networks, 
with regret that scales at most linearly in $d$ and polynomially in $s$.

\subsection{Notations}
For any positive integer $k$, we write $[k] := \{1,2,\ldots,k\}$. 
For any matrix $M$, $M_{i\cdot}$ and $M_{\cdot j}$ denote the $i$-th row and $j$-th column of $M$, respectively. For any $m\in \mathbb{N}_+$, we let $\one_m \in \mathbb{R}^m$ represent 
the $m$-dimensional vector of ones,
with the subscript sometimes dropped when no confusion can arise given the context.
For any vector $v \in \mathbb{R}^m$, we use $\supp(v)$ to denote the support of $v$;
for a set $A \subseteq [m]$, 
$v[A]$ refers to the subvector of $v$ corresponding to $A$.
For any set $A$, $|A|$ denotes its cardinality.
We use $\|M\|$ to denote the spectral norm of a matrix $M$
and $\|v\|_2$ to denote the Euclidean norm of a vector $v$.

\section{Problem formulation}\label{sec2}
We consider an experiment involving $d$ individuals over a time horizon of $T$ steps. 
At each time step $t \in [T]$, the experimenter assigns a treatment $a_{t,i}$ 
to individual $i \in [d]$. 
Here, the decisions are assumed to be sufficiently spaced, 
so that there are no significant carryover effects from previous time periods.\footnote{This is a common assumption in platforms that leverage switchback experiments, and can be relaxed to allow carryover effects over a known, small number of time periods~\citep[see, e.g.,][]{bojinov2023design}.} 
Letting $a_t := (a_{t,1},\ldots,a_{t,d})$ denote the vector of 
treatments at time $t$, we assume $a_t \in \mathcal{A}$, 
where $\mathcal{A} \in \{-1,1\}^d$. After administering the treatments, the experimenter observes the reward for 
each individual $i\in[d]$, modeled as 
\begin{align} 
\label{eq:model}
Y_{t,i} =   \sum_{j\in [d]} X_{ij}^* \cdot  a_{t,j} + \epsilon_{t,i},
\end{align}
where $\{\epsilon_{t,i}\}_{t\in[T], i\in[d]}$ are independent 
$1$-subgaussian noise terms and $X^*\in \mathbb{R}^{d \times d}$ 
is the {\em network treatment effect matrix} 
encoding both direct and indirect treatment effects.
In particular, the diagonal entry $X_{ii}^*$ 
represents the direct effect of treating individual $i$, 
while $X_{ij}^*$ captures the indirect, ``spillover'' effect 
of treating a different individual $j \neq i$ on $i$ 
(we assume without loss of generality that the baseline reward for 
no treatment is zero). 
Crucially, the matrix $X^*$ is not assumed to be symmetric;
that is, the influence of individual $i$ on individual $j$ 
may differ from that of $j$ on $i$, so $X_{ij}^*\neq X_{ji}^*$ in general.
For ease of notation, we let $Y_t := (Y_{t,1},\ldots,Y_{t,d})$ denote the vector of rewards 
observed at time $t$.

The above model assumes that additive treatment effects and,  
as discussed in~\citet{yu2022estimating}, 
is highly expressive, capable of capturing 
a broad range of network phenomena,
including spillover effects, peer effects, and contagion. 
It has also been widely used in the literature on treatment effect estimation under network interference (see, e.g.,~\citet*{toulis2013estimation,eckles2017design,sussman2017elements,hu2022average,yu2022estimating}).

We impose the following boundedness and sparsity assumptions on the model.
\begin{assumption}[Boundedness]
\label{asst:bounded}
The expected total rewards are uniformly bounded:
$\sup_{a \in \mathcal{A}}~ |X^*_{i\cdot}a|  \leq 1$,  $\forall i\in [d]$.
\end{assumption}
\begin{assumption}[Row sparsity]
\label{asst:sparsity}
For any $i\in[d]$, $\|X^*_{i\cdot}\|_0 \le s$ 
for some constant $s\ll \min(d,T)$.
\end{assumption}

The boundedness assumption is standard in the bandit literature~\citep{lattimore2020bandit}, 
and the upper bound $1$ is without loss of generality by rescaling the problem parameters if necessary. 
The sparsity assumption reflects the intuition that each individual can only be 
directly influenced by a small number $s$ of others, even as the population size $d$ grows larger. 
It is worth noting that the row-sparsity restriction only applies to the rows of $X^*$.
In particular, this allows for the presence of super-influencers—individuals whose actions affect many others—since no sparsity constraint is imposed on the columns.

In this work, we study an online setting in which the experimenter {\em adaptively}
assigns treatments to individuals in the population over a given time horizon $T$. 
Specifically, at each round $t \in [T]$, the experimenter chooses the treatment 
allocation $a_t$ based on the history of past actions and outcomes: 
$(a_1,Y_1),(a_2,Y_2),\ldots,(a_{t-1},Y_{t-1})$. 
Let $\pi = (a_1,a_2,\ldots,a_T)$ denote such an adaptive treatment assignment policy. 
The cumulative expected reward resulting from the policy $\pi$ is:
\$ 
\reward_{X^*}(\pi) := \sum^T_{t=1} \one^\top X^* a_t.
\$
To evaluate the performance of $\pi$, we compare its expected cumulative reward to 
that of the optimal policy. Specifically, letting
\[
a^* = \argmax{a\in\mathcal{A}}~ \one^\top X^* a
\]
denote the optimal allocation, we define the {\em cumulative regret} of a policy $\pi$ under $X^*$ 
to be
\[
\regret_{X^*}(\pi) = 
\sum_{t=1}^T \one^\top X^* (a^* - a_t). 
\]
Our goal is to design a treatment assignment policy $\pi$ minimizing 
the expected regret: $\EE[\regret_{X^*}(\pi)]$.

\subsection{Additional related work}

Prior work on policy learning under network interference spans several lines of literature, 
including causal inference with spillovers, online learning and bandits, 
and policy optimization in networks. We discuss the most relevant papers in each line of work and explain how our work relates to them.

\paragraph{Causal inference under interference.}
A large literature in causal inference studies 
inference of treatment effects in the presence 
of network interference (see e.g.,~\citet{ugander2013graph,banerjee2013diffusion,aronow2017estimating,basse2019randomization,leung2020treatment,li2022random,yu2022estimating,shirani2023causal}).
\citet{viviano2025policytargeting} studies welfare-maximizing policy targeting under network interference.
More recent work considers flexible representations of interference. For example, \citet{huber2026learningexposure} propose learning exposure mappings using graph neural networks, and \citet{schroder2026exposure} study causal inference under misspecified exposure mappings using partial identification. 
Our work differs from these papers by focusing on minimizing the cumulative 
regret in an online setting.

\paragraph{Online bandit algorithms with interference.}
In addition to~\citet{agarwal_multi-armed_2024},
there has also been growing interest in extending bandit algorithms to settings with interference across units. The settings in most of these papers are quite different from ours; in particular, they either focus on cases where the network is fully known or use very different reward or interference models. \citet{jia2024multiarmed} study multi-armed bandits with interference modeled through distances in a known graph, while \citet{xu2024linear} propose contextual bandit algorithms with interference assuming the network structure is known. 
\citet{jamshidi2025graph} analyze regret bounds for bandits with known graph-dependent interference. 
\citet{gleich2025scalablepolicy} study scalable policy optimization under network interference with dynamic networks.
Another line of work considers online learning underinference 
that balances regret and estimation accuracy given an exposure mapping~\citep{zhang2024onlineexpdesign,wang2025designbasedbandits}.

\paragraph{Online bandit algorithms without interference.}
There is a long list of work on contextual bandits without interference \citep{abbasi2012online, carpentier2012bandit, wang2018minimax, kim2019doubly, bastani2020online, lattimore2015linear,  ren2023dynamic, hao2020high}. 

In particular, for high-dimensional linear bandit, \citet{hao2020high} propose an algorithm achieve regret of $\tilde{O}((sT)^{2/3})$;
our work generalizes their algorithm to the network interference setting.

\paragraph{Adaptive seeding.} There is another line of work known as adaptive seeding, which optimally selects individuals (called ``seeds'') to trigger the spread of information, behaviors, or products \citep[e.g.][]{seeman2013adaptive, akbarpour2020just}. Although this problem setting appears to be very similar to our work, 
it has a very different goal---maximizing the total spread of influence across a social network, 
whereas our work aims at maximizing the overall cumulative reward. These goals do not necessarily align in the presence of heterogeneous or negative treatment effects.

\section{Baseline: reduction to network-agnostic linear bandit}\label{sec3}
Given the problem, it is natural to consider a na\"{i}ve approach 
that reduces it to a standard linear bandit with the unknown parameters 
$\theta := (X^*)^\top \one$. 
In the reduced model, the reward can be expressed as 
\begin{align}
\label{eq:model_baseline}
Z_t = \one^\top Y_t = (\one^\top X^*)  a_t + \one^\top \epsilon_t 
= : \theta^\top a_t + \eta_t, 
\end{align}
where $\eta_t = \one^\top \epsilon_t$ is a $d$-sub-Gaussian random variable. 

For this reduced problem, we are essentially required to design 
an adaptive treatment allocation strategy where at each time $t$, 
the action $a_t$ is adapted to the filtration of $(a_1,Z_1,\ldots, a_{t-1}, Z_{t-1})$. 
This problem can be solved with an off-the-shelf linear bandit algorithm~\citep{abbasi2011improved}. 

Note that our sparsity assumption (Assumption~\ref{asst:sparsity}) is imposed on the rows, 
while the column sum $\theta$ is not necessarily sparse. As a result, the sparsity structure 
cannot be exploited in the reduced problem.
In what follows, we prove a lower bound formalizing this intuition, 
showing that the regret with this strategy must be super-linear in $d$ in the worst case 
(specifically, $\Omega(d^{3/2})$), which is unsatisfactory in high-dimensional settings. As a consequence, to achieve better regret, we must make use of the individual outcomes $Y_t$ rather than aggregating them in this way.

\subsection{The baseline is inefficient: regret lower bound}
As discussed earlier, when we reduce the task to a linear bandit problem, 
the algorithm essentially operates on $(a_t, Z_t)_{t=1}^T$, 
where $Z_t$ is the overall reward defined in~\eqref{eq:model_baseline}. 
Formally, we consider any policy $\pi$ such that at each time $t$, 
$a_t \in \sigma(a_1,Z_1,\ldots,a_{t-1},Z_{t-1})$, 
where $\sigma(\cdot)$ denotes the sigma-algebra generated by the random variables inside the parentheses.

Before stating the lower bound, we set up some notation. Let $S\in \RR^{d\times d}$ represent 
the support of the network treatment effect matrix $X^*$, i.e., 
$S_{ij} = \ind\{X^*_{ij} \neq 0\}$, $\forall i,j \in [d]$. 
For any column $j\in[d]$, we define its support size as $\rho_j = \sum^d_{i=1} S_{ij}$.
Given 
$S$, we define $\mathbb{M}(S)$ as the set of all network treatment effect matrices 
that share the same support $S$ and satisfy Assumption~\ref{asst:bounded}.

Theorem~\ref{thm:instance-baselineLB} below provides a lower bound on the regret for any policy
operating on the reduced information $(a_t,Z_t)_{t=1}^T$. 
We note that the lower bound is in the {\em instance-dependent} form---with explicit dependence on 
the support set of $X^*$---which is strictly stronger than the minimax lower bound.

\begin{theorem}\label{thm:instance-baselineLB}
Suppose the noise terms $\{\epsilon_{t,i}\}_{t\in[T],i\in[d]}$ are 
i.i.d.~from a standard normal distribution.
Given a support specification $S$ satisfying Assumption~\ref{asst:sparsity},
for any policy $\pi$ such that $a_t \in \sigma(a_1,Z_1,\ldots,a_{t-1},Z_{t-1})$, $\forall t\in[T]$, 
we have 
\[
\sup_{X \in \mathbb{M}(S)} \EE\big[\regret_{X}(\pi)\big] \geq 
\frac{\sqrt{T}}{16}\sum_{j=1}^d
\min \Big\{ \sqrt{d}\cdot \ind\{\rho_j > 0\}, 
 \frac{2\sqrt{T}}{s} \rho_j\Big\},
\]
where $\mathbb{M}(S)$ is the set of network treatment effect matrices
sharing the same support $S$ and
satisfying Assumption~\ref{asst:bounded}. 
\end{theorem}

The proof of Theorem~\ref{thm:instance-baselineLB} is provided in Appendix~\ref{app:baselineLB}.
We can see that the lower bound crucially depends on the support configuration of $X^*$ 
through the column-support profile $\rho:=(\rho_1,\dots,\rho_d)$.
Under the sparsity assumption (Assumption~\ref{asst:sparsity}), we have 
$\|\rho\|_1 = \sum_{i,j=1}^d S_{ij} \le ds$, while the number of nonzero columns $\|\rho\|_0$
can still vary significantly depending on how the nonzero entries in $S$ are distributed.
In the worst case, we can have $\|\rho\|_0 = d$, which happens when each column has at least one non-zero entry. 
The following corollary realizes this worst-case scenario, showing that 
for any reduced algorithm, there exists a problem instance for which the regret is 
$\Omega(d^{3/2}\sqrt{T} \wedge dT)$. 
\begin{corollary}
\label{cor:minimax-baselineLB}    
Suppose the noise terms $\{\epsilon_{t,i}\}_{t\in[T],i\in[d]}$ 
are i.i.d.~from a standard normal distribution.
For any policy $\pi$ such that $a_t \in \sigma(a_1,Z_1,\ldots,a_{t-1},Z_{t-1})$, $\forall t\in[T]$,  
there exists a network treatment effect matrix $X$ satisfying Assumptions~\ref{asst:bounded} and~\ref{asst:sparsity},
such that 
\[
\EE\big[\regret_{X}(\pi)\big] \geq 
\frac{\sqrt{T}}{16}\cdot \min \Big\{d^{3/2}, 
2d\sqrt{T}\Big\}.
\]
\end{corollary}
\begin{proof}
The corollary is proved by taking $S$ in Theorem~\ref{thm:instance-baselineLB}
to be the circulant $s$-regular matrix: 
\$ 
S_{ij} = \ind\{(j-i) \text{ mod } d < s\}, \forall i,j \in [d].
\$

This choice satisfies the row-sparsity assumption (Assumption~\ref{asst:sparsity})
and yields $\rho_j = s$, $\forall j\in[d]$. 
Substituting these into the lower bound in Theorem~\ref{thm:instance-baselineLB} 
yields the desired result.
\end{proof}

By Corollary~\ref{cor:minimax-baselineLB}, we see that the regret of any algorithm
is lower bounded by $\Omega(d^{3/2}\sqrt{T})$ when $d\le T$ 
or by $\Omega(dT)$ if $d > T$. 
That is, the regret is either superlinear in $d$ or linear in $T$. 
This shows that the na\"{i}ve reduction to a standard linear bandit
is highly inefficient in high-dimensional settings, 
motivating the need for more sophisticated algorithms that can leverage the network structure.

\section{Our proposal: leveraging the network structure}
\label{sec4}
As we have seen in Section~\ref{sec3}, na\"{i}vely reducing our problem to a standard linear bandit
loses important information about the network structure, leading to unsatisfactory regret performance.
In this section, we discuss how to leverage the row-wise observations $Y_t$ to design more efficient algorithms, 
i.e., the action $a_t$ is now adapted to 
the entire history $(a_1,Y_1,\ldots,a_{t-1},Y_{t-1})$, where we recall that $Y_t$ 
is the reward vector $(Y_{t,1}\dots,Y_{t,d})$.

Below, we first present a lower bound on the regret for any policy that utilizes the full information,
which reveals the fundamental limits of learning under network interference.
We then propose algorithms that efficiently exploit the network structure 
with near-optimal regret performance.

\subsection{Regret lower bound}
The lower bound for the full-information setting is similarly in the per-instance fashion.
Given any support specification $S$ satisfying Assumption~\ref{asst:sparsity}, 
we can find a network treatment effect matrix $X^*$ supported on $S$ and satisfying Assumption~\ref{asst:bounded}
such that the regret of any policy $\pi$ is lower bounded in terms of $(d,s,T)$ and column-support profile $\rho$.

\begin{theorem}\label{thm:LB} 
Suppose the noise terms $\{\epsilon_{t,i}\}_{t\in[T],i\in[d]}$ 
are i.i.d.~from a standard normal distribution.
 Given a support specification $S$ satisfying Assumption~\ref{asst:sparsity},
for any policy $\pi$ such that $a_t \in \sigma(a_1,Y_1,\dots,a_{t-1},Y_{t-1})$, $\forall t \in [T]$, we have 
\begin{align*}
\sup_{X \in \mathbb{M}(S)} \EE\big[\regret_{X}(\pi) \big] \geq \frac{1}{16} \cdot\sum_{j=1}^d \min\Big\{\sqrt{T\rho_j}, \frac{2T\rho_j}{s}\Big\},
\end{align*}
where $\mathbb{M}(S)$ is the set of network treatment effect matrices sharing the same support $S$
and satisfying Assumption~\ref{asst:bounded}. 
\end{theorem}
The proof of Theorem~\ref{thm:LB} is deferred to Appendix~\ref{app:LB}.
Not surprisingly, the lower bound also depends crucially on the support configuration of 
the network treatment effect matrix. In particular, the quantity
$\sum_{j=1}^d \sqrt{\rho_j} \le (\sum^d_{j=1}\rho_j)^{1/2}\sqrt{d} \le d\sqrt{s}$, by
the Cauchy-Schwarz inequality, and the term $\sum \rho_j \le ds$. 
This means that the lower bound is---in the 
worst case---$\Omega(d\sqrt{sT})$. 
This intuition is formalized by the following corollary.
\begin{corollary}\label{cor:worst-case-LB}
Suppose the noise $\epsilon_{t,i}$ is i.i.d.~from a standard normal distribution. 
For any policy $\pi$ such that $a_t \in \sigma(a_1,Y_1,\dots,a_{t-1},Y_{t-1})$, $\forall t \in [T]$, 
there exists a network treatment effect matrix $X$ satisfying Assumptions~\ref{asst:bounded} 
and \ref{asst:sparsity} such that
\begin{align*}
\EE\big[\regret_{X}(\pi)\big] \geq \frac{d\sqrt{sT}}{16}.
\end{align*}
\end{corollary}
\begin{proof}
The corollary is proved by taking $S$ in Theorem~\ref{thm:LB} to be 
the circulant $s$-regular matrix: 
\$ 
S_{ij} = \ind\{(j-i) \text{ mod } d < s\}, \forall i,j \in [d].
\$
With this choice, row sparsity is $s$ and we have $\rho_j = s$, $\forall j \in [d]$;
we can also verify that $\|\rho\|_1 = ds$. 
As a result, the lower bound in Theorem~\ref{thm:LB} is at least 
\[
\frac{\sqrt{T}}{16} \cdot \min \big\{d\sqrt{s}, 2\sqrt{T}d \big\} = \frac{d\sqrt{sT}}{16},
\] 
where the equality uses the assumption that $s \le T$. The desired result follows.
\end{proof}

Compared to the previous worst-case lower bound $\Omega(\min\{d\sqrt{dT},dT\})$ 
established in the reduced setting of Section~\ref{sec3}, the new lower bound
 in Corollary~\ref{cor:worst-case-LB}
is significantly smaller when the sparsity level $s$ is much smaller than $d$. 
This implies that---with more 
information---it is in principle possible to design a treatment allocation policy whose regret 
averaged over the population does not scale with $d$ and is sublinear in $T$. 
The remaining challenge is to design such algorithms achieving this possibility.
We consider three scenarios with different levels of knowledge about the support of the 
network treatment effect matrix: (i) full support knowledge, (ii) partial support knowledge, 
and (iii) no support knowledge (apart from the row sparsity level).

\subsection{Algorithm with full support knowledge}
In many applications, the support of the network structure is known to the experimenter a priori.
For example, on online social network platforms, the experimenter can 
observe the connections between users;
in studies with clustered units, the experimenter knows which units belong to the same cluster.
In such settings, it is natural to leverage the full support information of the network to 
design an efficient treatment allocation algorithm.

Our proposed algorithm for the full-support setting
is a batched, {\em successive elimination}-type algorithm. 
Intuitively, the optimal action for $j \in [d]$ is the sign of the corresponding column sum of 
the network treatment effect matrix, 
i.e., $\text{sign}(\theta_j)$, where we recall $\theta_j = \sum_{i=1}^d X_{ij}$. 
The task then boils down to estimating $\sign(\theta_j)$ for each column $j \in [d]$.
To this end, our proposed algorithm takes actions uniformly sampled from $\{\pm 1\}$ 
for each coordinate while maintaining a confidence interval for each $\theta_j$.
At the end of each batch, it eliminates the suboptimal actions once the confidence 
interval for $\theta_j$ excludes zero
and commits to the estimated optimal action for the remaining rounds. 

To be concrete, we consider a sequence of $M = \lceil \log_2(T/2+1)\rceil$ batches 
indexed by $m = 1,\ldots,M$, with the size of the $m$-th batch being (at most) $2^m$.
The $m$-th batch consists of rounds $t = T_{m-1}+1,\ldots,T_m$, where 
$T_m = 2(2^m-1)$, for $m \in [M-1]$, and $T_0 = 0$, $T_M = T$. 
For notational convenience, we let $D_m = \{T_{m-1}+1,\dots,T_m\}$ denote the set of rounds in the $m$-th batch:
\$
\underbrace{1,\ldots,T_1}_{D_1},~\ldots,~
~\underbrace{T_{m-1}+1,\ldots,T_m}_{D_m},~
\ldots,~\underbrace{T_{M-1}+1,\ldots,T}_{D_M}, 
\$
and let $n_m = |D_m|$ denote the batch sizes.
For any $m\ge 0$, we construct an estimator for $\theta$, denoted by $\hat \theta^{(m)}$, 
based on the observations collected in the $m$-th batch;
we use a set $U_m \subseteq [d]$ to denote the indices with undetermined $\sign(\theta_j)$ 
at the end of batch $m$.

At initialization, we set $U_0 = [d]$ and $\hat \theta^{(0)}_j = 0$ for all $j \in [d]$.
Then for any $m \in [M]$, we choose the action 
\$ 
a_{t,j} = \begin{cases}
R_{t,j}  & \text{if } j \in U_{m-1},\\
\sign(\hat \theta_j^{(m-1)}) & \text{otherwise},
\end{cases} 
\$
for any $t \in D_m$, where $R_{t,j} \sim \text{Unif}\{\pm 1\}$ is drawn independently.
At the end of the $m$-th batch, we gather the observations and update the estimator for 
$\theta$ based on the data collected in the $m$-th batch, 
and then update the undetermined set $U_m$ 
by eliminating the indices whose signs are confidently estimated.
For the estimator construction and the elimination rule, we adopt 
two different strategies for the warm-up phase (i.e., $m < m_0$) 
and the ordinary least squares (OLS) phase (i.e., $m \ge m_0$), 
where $m_0$ is a tuning parameter to be specified later.

\paragraph{Warm-up phase.}
For batch $m < m_0$, we construct an estimator for $X_{ij}^*$ as 
\@ \label{eq:fs-matrix-entry}
\hat X_{ij}^{(m)} = \frac{1}{n_m}\sum_{t \in D_m} Y_{t,i} a_{t,j}, ~\forall i,j \in [d],
\@
and then estimate $\theta_j$ by 
\@ \label{eq:fs-theta-estimator}
\hat \theta_j^{(m)} = \sum_{i=1}^d \hat X_{ij}^{(m)} \ind\{j \in S_i\}, ~\forall j \in [d],
\@
where we slightly abuse the notation and let $S_i$ denote the support of the $i$-th row of $X^*$.
We also update the undetermined set $U_m$ as 
\@ \label{eq:fs-undetermined-set}
U_m = \big\{j \in U_{m-1}: |\hat \theta_j^{(m)}| \le \rho_j\tau_{m}\big\},
\@
where $\tau_{m}>0$ is a threshold to be specified later.

\paragraph{OLS phase.}
For batch $m \ge m_0$, 
we similarly gather the observations collected in the $m$-th batch,
but now update the estimator for $X^*_{i,\cdot}$ 
via linear regression restricted to the \emph{intersection} of the support and the undetermined set 
$E_i^{(m)} := S_i \cap U_{m-1}$. Specifically, for each row $i\in[d]$, we solve the following problem: 
\$ 
\hat X_{i,\cdot}^{(m)} = \argmin{\gamma \in \RR^d, \gamma[(E_i^{(m)})^c] = 0}~\frac{1}{2n_m}
\sum_{t \in D_m} \big\{Y_{t,i} - \bar Y_{i}^{(m)} - \gamma^\top (a_t - \bar a^{(m)})\big\}^2,
\$
where we now write $\bar Y_{i}^{(m)} = \frac{1}{n_m}\sum_{t \in D_m} Y_{t,i}$ 
and $\bar a^{(m)} = \frac{1}{n_m}\sum_{t \in D_m} a_t$ as the batch-wise 
sample averages of the rewards and actions, respectively.
The estimator for $\theta_j$ is then updated for each $j \in U_{m-1}$ via: 
\@\label{eq:theta-update} 
\hat \theta_j^{(m)} = \sum_{i=1}^d \hat X_{i,j}^{(m)} \ind\{j \in S_i\}.
\@
Finally, the uncertainty set is updated as
\@\label{eq:full-support-active}
U_{m} = \big\{j \in U_{m-1}: |\hat \theta_j^{(m)}| \le \sqrt{\rho_j}\tau_{m}\big\}.
\@
We refer to the above policy as $\pi_{\nsefs}$ and 
describe it in detail in Algorithm~\ref{alg:full-support}.

\begin{algorithm}[h!]
\caption{Network successive elimination with full support knowledge (NSE-FS)}
\label{alg:full-support}
\DontPrintSemicolon 
\SetAlgoLined 
\KwIn{Horizon $T$, support specification matrix $S$, threshold parameters $\{\tau_m\}_{m=1}^M$, 
warm-up batch number $m_0$.}
Initialize $T_0=0$, $U_0 = [d]$, $\hat \theta_j^{(0)} = 0$, $\forall j \in [d]$.\;
{\color{blue} \tcp{Warm-up phase.}}
\For{$m = 1$ to $m_0-1$}{
    Set $T_m = \min\{2(2^m-1),T\}$.\;
     \For{$t = T_{m-1}+1$ to $T_m$}{
        \For{$j=1$ to $d$}{
            \eIf{$j \in U_{m-1}$}{
                Sample $a_{t,j} \in \{\pm 1\}$ uniformly at random.\;
            }{
                Set $a_{t,j} = \sign(\hat \theta_j^{(m-1)})$.\;
            }
        }
        Play action $a_t = (a_{t,1},\ldots,a_{t,d})$ and observe reward vector $Y_t$.\;
    }
    \For{$j \in U_{m-1}$}{
        Construct $\hat \theta_j^{(m)}$ according to Equation~\eqref{eq:fs-theta-estimator}.\;
    }
    Update $U_{m}$ according to Equation~\eqref{eq:fs-undetermined-set}.\;
}
{\color{blue} \tcp{OLS phase.}}
\For{$m = m_0$ to $M$}{
    Set $T_m = \min\{2(2^m-1),T\}$.\;
    \For{$t = T_{m-1}+1$ to $T_m$}{
        \For{$j=1$ to $d$}{
            \eIf{$j \in U_{m-1}$}{
                Sample $a_{t,j} \in \{\pm 1\}$ uniformly at random.\;
            }{
                Set $a_{t,j} = \sign(\hat \theta_j^{(m-1)})$.\;
            }
        }
        Play action $a_t = (a_{t,1},\ldots,a_{t,d})$ and observe reward vector $Y_t$.\;
    }
    \For{$j \in U_{m-1}$}{
        Update $\hat \theta_j^{(m)}$ according to Equation~\eqref{eq:theta-update}.\;
    }
    Update $U_{m}$ according to Equation~\eqref{eq:full-support-active}.\;
}
\end{algorithm}

Theorem~\ref{thm:full-support} below shows that with carefully chosen 
warm-up length and thresholds, we can bound the regret of $\pi_{\nsefs}$
with high probability. The proof of Theorem~\ref{thm:full-support} can 
be found in Appendix~\ref{app:full-support}. 
\begin{theorem}[High-probability regret upper bound of NSE-FS]
\label{thm:full-support}
Suppose Assumptions~\ref{asst:bounded} and \ref{asst:sparsity} hold, and the support specification $S$ is known to the algorithm.
Assume that $T \ge ds$.
For any $\delta \in (0,1)$, if we choose the warm-up batch number $m_0 = \ceil{\log_2(128s\log(8\log_2 Tds/\delta))} \wedge M$,
and set the threshold parameters to be $\tau_m = 8\sqrt{\frac{\log(16d^2\log_2 T/\delta)}{2^m}}$ for all $m \in [M]$, 
then
with probability at least $1-\delta$, the regret of $\pi_{\nsefs}$ satisfies
\begin{align*}
\regret_{X^*}(\pi_\nsefs) & \leq  c_1 \log\bigg(\frac{16d^2s\log_2 T}{\delta}\bigg)
\cdot  \sqrt{T}\sum^d_{j=1}  \sqrt{\rho_j},
\end{align*}
where $c_1 > 0$ is a numerical constant.
\end{theorem}

Based on the high-probability regret bound in Theorem~\ref{thm:full-support}, we can also derive a bound on the expected regret of $\pi_{\nsefs}$, as stated in the following corollary.
The proof of Corollary~\ref{cor:full-support} is completed by taking $\delta = 1/(dT)$, and 
we provide it in Appendix~\ref{app:full-support-expected} for completeness.
\begin{corollary}[Expected regret upper bound of $\pi_{\nsefs}$]
\label{cor:full-support}
Under the same conditions as in Theorem~\ref{thm:full-support}, the expected regret of $\pi_{\nsefs}$ satisfies
\begin{align*}
\EE\big[\regret_{X^*}(\pi_\nsefs)\big] & \leq  c_2 \log\big(16d^3Ts\log_2 T\big) \cdot \sqrt{T}\sum^d_{j=1}  \sqrt{\rho_j},
\end{align*}
where $c_2 > 0$ is a numerical constant.
\end{corollary}

Note that when $T\ge ds$, the lower bound in Theorem~\ref{thm:LB} is $\Omega(\sqrt{T}\sum_{j=1}^d \sqrt{\rho_j})$, 
which matches the upper bound in Corollary~\ref{cor:full-support} up to logarithmic factors.
In other words, $\pi_{\nsefs}$ is minimax-optimal in terms of the expected regret when the full support information is available.

\subsection{Algorithm with partial support knowledge}
We now consider the case where only the {\em column support sizes} $\rho_j$ are known to the experimenter.
This is a weaker assumption than knowing the full support of the network structure, 
as it only requires the experimenter to know the number of nonzero entries in each column of $X^*$---the number of units influenced by each unit---without needing to know their exact locations.
Such an assumption is still reasonable in many applications, as the experimenter may have some prior knowledge of the network structure, e.g., in social network platforms, the experimenter may know the number of connections each user has, even if they do not know the specific connections;
or in studies with privacy concerns, the experimenter may only have access to the degree information of the network but not the full adjacency matrix.

For this partial-support setting, we propose a modified version of the successive elimination algorithm,
where the main difference lies in the construction of the estimator for $\theta$ and the elimination decision. 
Similar to the full-support case, we also divide the entire horizon into $M = \lceil \log_2(T/2+1)\rceil$ batches,
with the $m$-th batch consisting of rounds $D_m = \{T_{m-1}+1,\ldots,T_m\}$, where $T_m = 2(2^m-1)$ for $m \in [M-1]$, and $T_0 = 0$, $T_M = T$.
At each batch, we construct an estimator for $\theta$ and 
maintain an uncertainty set $U_m \subseteq [d]$, initialized at $U_0 = [d]$ and $\hat \theta_j^{(0)} = 0$ for all $j \in [d]$.

For $m \ge 1$ and any $t \in D_m$, let 
\$ 
a_{t,j} = \begin{cases}
R_{t,j}  & \text{if } j \in U_{m-1},\\
\sign(\hat \theta_j^{(m-1)}) & \text{otherwise},
\end{cases}
\$
where $R_{t,j} \in \{\pm 1\}$ is drawn uniformly at random.
At the end of the $m$-th batch, we update the estimator for $X^*_{ij}$: 
\$ 
\hat X_{ij}^{(m)} = \frac{1}{n_m}\sum_{t \in D_m} Y_{t,i} a_{t,j}, 
~\forall  i\in[d] \text{ and } j\in U_{m-1}
\$

For each $j\in U_{m-1}$, the estimator for $\theta_j$ is then given by a hard-thresholding of the column sum of $\hat X^{(m)}$: 
\$ 
\hat \theta_j^{(m)} := \sum^d_{i=1} 
\hat X_{ij}^{(m)}\cdot \ind\Big\{|\hat X_{ij}^{(m)}| > \tau_{m}/8\Big\},
\$
where $\tau_{m} > 0$ is a threshold to be specified. The uncertainty set is updated as
\$
U_{m} = \big\{j \in U_{m-1}: |\hat \theta_j^{(m)}| \le \rho_j\tau_{m}\big\}.
\$
The complete algorithm is summarized in Algorithm~\ref{alg:col-support}.

\begin{algorithm}[ht]
\caption{Network successive elimination (NSE)}
\label{alg:col-support}
\DontPrintSemicolon 
\SetAlgoLined 
\KwIn{Horizon $T$, column support sizes $\{\rho_j\}_{j=1}^d$, 
threshold parameters $\{\tau_m\}_{m=1}^M$.} 

Initialize $T_0=0$, $U_0 = [d]$, $\hat \theta_j^{(0)} = 0$, $\forall j \in [d]$.\;
\For{$m = 1$ to $M$}{
    Set $T_m = \min\{2(2^m-1),T\}$.\;
    \For{$t = T_{m-1}+1$ to $T_m$}{
        \For{$j=1$ to $d$}{
            \eIf{$j \in U_{m-1}$}{
                Sample $a_{t,j} \in \{\pm 1\}$ uniformly at random.\;
            }{
                Set $a_{t,j} = \sign(\hat \theta_j^{(m-1)})$.\;
            }
        }
        Play action $a_t = (a_{t,1},\ldots,a_{t,d})$ and observe reward vector $Y_t$.\;
    }
    \For{$j \in U_{m-1}$}{
        \For{$i=1$ to $d$}{
            Update $\hat X_{ij}^{(m)} = \frac{1}{n_m}\sum_{t \in D_m} Y_{t,i} a_{t,j}$.\;
        }
        Update $\hat \theta_j^{(m)} = \sum^d_{i=1} 
        \hat X_{ij}^{(m)}\cdot \ind\big\{|\hat X_{ij}^{(m)}| > \tau_{m}/8\big\}$.\; 
    }
    Update $U_{m} = \big\{j \in U_{m-1}: |\hat \theta_j^{(m)}| \le \rho_j\tau_{m}\big\}$.\;
}
\end{algorithm}

We refer to the above algorithm as $\pi_\nse$ and 
the following theorem provides a high-probability upper bound on the regret of 
$\pi_{\mathsf{nse}}$ in terms of time horizon $T$, the size of the network $d$,
and the column support sizes $\rho_j$. Its proof 
provided in Appendix~\ref{app:se_upper_bound}.

\begin{theorem}[High-probability regret upper bound] 
\label{thm:se_upper_bound}
Suppose Assumptions~\ref{asst:bounded} and~\ref{asst:sparsity} hold.
For any $\delta \in (0,1)$, if the threshold parameters are chosen as $\tau_m = 16\sqrt{\frac{\log(4d^2\log_2(T)/\delta)}{2^{m-1}}}$ for $m \in [M]$, then
 with probability at least $1 - \delta$, 
the regret of Algorithm~\ref{alg:col-support} satisfies 
\$
\regret_{X^*}(\pi_{\nse}) \le c_3\sqrt{T\log(4d^2\log_2(T)^2/\delta)}\sum^d_{j=1} \rho_j, 
\$ 
where $c_3 > 0$ is a numerical constant.
\end{theorem}
As a corollary of Theorem~\ref{thm:se_upper_bound}, we have the following expected regret upper bound
by taking $\delta = 1/(dT)$. We incorporate the proof in Appendix~\ref{app:se_upper_bound} for completeness.
\begin{corollary}[Expected regret upper bound] 
\label{cor:se_upper_bound}
Under the same conditions as in Theorem~\ref{thm:se_upper_bound}, 
when the threshold parameters are chosen as $\tau_m = 16\sqrt{\frac{\log(4d^3T\log_2(T))}{2^{m-1}}}$ for $m \in [M]$, 
the expected regret of $\pi_{\nse}$ satisfies 
\$
\EE\big[\regret_{X^*}(\pi_{\nse})\big] \le c_4\sqrt{T\log(4d^2 T \log_2(T))}\sum^d_{j=1} \rho_j,
\$ 
where $c_4 > 0$ is a numerical constant.
\end{corollary}
As we can see in Theorem~\ref{thm:se_upper_bound} and Corollary~\ref{cor:se_upper_bound},
the regret upper bounds scale with $\sqrt{T}\sum_{j=1}^d \rho_j$ up to logarithmic factors.
This matches the lower bound in Theorem~\ref{thm:LB} up to the remaining $\sqrt{\rho_j}$ gap and logarithmic factors. 
Importantly, under Assumption~\ref{asst:sparsity}, we have $\sum_{j=1}^d \rho_j \le ds$,
so the regret upper bounds are at most on the order of $ds\sqrt{T}$---the dependence on $d$ is linear,
and the dependence on $T$ is sublinear. This shows that using the network structure 
can significantly reduce the regret compared to the baseline algorithm in Section~\ref{sec3}.

\subsection{Algorithm without support information}

We now consider the setting where no support information of the network is available
beyond the row-sparsity level $s$. In this case, we adopt an explore-then-commit strategy, 
where the algorithm first explores the action space to estimate the network structure and then commits to the best action 
based on the estimated network.

To be concrete, a time period of length $T_1 = T^{2/3}s^{2/3}$ is first allocated for exploration. 
For $t \le T_1$, the algorithm samples actions $a_{t,j} \stackrel{\text{i.i.d.}}{\sim} \text{Unif}\{\pm 1\}$,
for $j\in[d]$. At the end of the exploration phase, 
for each row $i \in [d]$, the algorithm computes the 
LASSO estimator 
for estimating $X^*_{i\cdot}$:
\@\label{eq:row_lasso}
\hat{X}_{i\cdot} = \underset{\gamma \in \mathbb{R}^d}{\mathrm{argmin}}~ 
\bigg\{\frac{1}{2T_1} \sum_{t=1}^{T_1}(Y_{t, i} - \gamma^\top a_t)^2 + \lambda \|\gamma\|_1\bigg\},    
\@
where $\lambda > 0$ is a regularization parameter to be specified later.
The column sum estimator $\hat \theta_j$ is then given by $\hat \theta_j = \sum_{i=1}^d \hat X_{ij}$.
In the exploitation phase, for each round $t = T_1+1,\ldots,T$, the algorithm plays the action 
$\hat a = \sign(\hat \theta)$.

\begin{algorithm}[ht]
\caption{Network explore-then-commit (NETC)}\label{alg:netc}
\DontPrintSemicolon 
\SetAlgoLined 
\KwIn{Time horizon $T$, sparsity level $s$, regularization parameter $\lambda$.}
\tcp{\color{blue} Exploration phase.}
Set $T_1 = T^{2/3}s^{2/3}$.\;
\For{$t = 1$ to $T_1$}{
Sample $\displaystyle a_{t,j} \sim \text{Unif}\{\pm 1\},~\forall j\in[d]$\;
Play action $a_t$ and observe reward vector $Y_{t}$.\;
}
\For{$i = 1$ to $d$}{
Obtain $\hat X_{i\cdot}$ according to Equation~\eqref{eq:row_lasso}.\;
}

\tcp{\color{blue} Commit phase.}
\For{$t = T_1+1$ to $T$}{
Play action $a_t = \sign(\hat \theta)$ and observe reward vector $Y_{t}$.\;
}
\end{algorithm}

We refer to this algorithm as $\pi_\netc$ and 
summarize the complete algorithm in Algorithm~\ref{alg:netc}. 
Theorem~\ref{thm:dpUB} provides a high-probability upper bound on its regret
followed by a corollary describing the expected regret bound.

\begin{theorem}\label{thm:dpUB}
Suppose Assumptions~\ref{asst:bounded} and~\ref{asst:sparsity} hold.
For any $\delta \in(0,1)$, if $T \gtrsim s^{1/2}\log(d/\delta)^2$
and the regularization parameter is chosen as $\lambda = 4\sqrt{\frac{2\log(2d^2/\delta)}{T^{2/3}s^{2/3}}}$,
then 
with probability at least $1-\delta$, the regret of Algorithm~\ref{alg:netc} satisfies
\begin{align*}
\regret_{X^*}(\pi_{\netc}) \leq c_5 d (sT)^{2/3}\sqrt{2\log(2d^2/\delta)},
\end{align*}
where $c_5$ is a numerical constant.
\end{theorem}

\begin{corollary}\label{cor:dpUB}
Under the same conditions as in Theorem~\ref{thm:dpUB}, the expected regret of Algorithm~\ref{alg:netc} satisfies
\begin{align*}
\EE\big[\regret_{X^*}(\pi_{\netc})\big] \leq c_6 d (sT)^{2/3}\sqrt{2\log(2d^2T)},
\end{align*}
where $c_6$ is a numerical constant.    
\end{corollary}
The proofs of Theorem~\ref{thm:dpUB} and Corollary~\ref{cor:dpUB} are 
provided in Appendix~\ref{app:dpUB} and~\ref{app:cor-dpUB}, respectively. 
Our proof extends the explore-then-commit procedure in 
\citet{hao2020high} to our network setting. 
Compared to the algorithms with structural knowledge that achieve $\sqrt{T}$-type regret, the regret here is inflated by $T^{1/6}$.
This difference
has also been observed by \citet{hao2020high} when there is no interference, suggesting that $T^{2/3}$ may be the best we can achieve without support information. Finally, when $T < d^3$, this algorithm will outperform the baseline; intuitively, the baseline algorithm requires a long exploration phase since it is solving a non-sparse $d$-dimensional bandit problem.

\section{Experiments}\label{sec5}

We evaluate the empirical performance of our algorithms using simulated networks 
and semi-synthetic data constructed from social networks in Indian villages,
collected by \citet{banerjee2013diffusion}. Throughout, we compare four methods:
\begin{enumerate}
\item[(i)] a baseline that applies the {\em upper confidence bound (UCB)} 
algorithm to the aggregated reward~\citep{abbasi2011improved};
\item[(ii)] NSE-FS (Algorithm~\ref{alg:full-support}), which exploits full support information;
\item[(iii)] NSE (Algorithm~\ref{alg:col-support}), which uses column support sizes; 
\item[(iv)] NETC (Algorithm~\ref{alg:netc}), which uses only row sparsity.
\end{enumerate}
\paragraph{Implementation details.} Optimization problems are solved with Gurobi. For NETC, we set the Lasso regularization to $\lambda=0.035$ and the exploration length to $T_1=200$ for the simulated experiments and $T_1=300$ for the village experiment. For NSE, we use the one-hot estimator with $\tau_m=c_\tau\sqrt{2\log(2T)/T_m}$ and $c_\tau=0.2$. For NSE-FS, we estimate $\widehat X^{(m)}$ using ordinary least squares and apply hard-threshold elimination with threshold parameter $\delta=0.05$ and threshold constant $8.0$.

\subsection{Experiments with simulated data}\label{sec:simulated_data}

First, we perform experiments using a randomly generated network. 
We fix the action set to be the hypercube $\cA = [-1, 1]^d$, 
and run simulations with $d = 100$ and $T = 20000$. 

The networks and heterogeneous signals are generated using a mixed signal model. For any $i,j\in[d]$, 
we first generate the latent variables:
\[
Z_{ij} \sim \mathrm{Unif}(-1,1),~ 
U_{ij} \sim \mathrm{Unif}(0,1),~R_{ij} \sim \mathrm{Unif}\{\pm 1\},
\]
and then define $X_{ij}$ to be
\begin{equation}\label{eq:signal_function}
X_{ij}
=
\begin{cases}
0, 
& \text{if } i \neq j \text{ and } U_{ij} > s_0/d, \\[6pt]
\beta \cdot Z_{ij}, 
& \text{if } \big(i=j \text{ or } U_{ij} \le s_0/d \big)
\text{ and } R_{ij}=1, \\[6pt]

0.001\beta \cdot Z_{ij}, 
& \text{if } \big(i=j \text{ or } U_{ij} \le s_0/d\big)
\text{ and } R_{ij}=-1,
\end{cases}
\end{equation}
where $\beta=0.1$ determines the signal strength and $s_0=20$ determines the expected row-sparsity. The outcome is generated as in \eqref{eq:model} with $\epsilon_{t,i} \stackrel{\text{iid}}{\sim} N(0, 1)$, $\forall t\in[T], i\in[d]$.
Additional simulation results with varying values of $\beta$ and $s_0$
are deferred to Appendix~\ref{app:experiments}.

We run the four candidate algorithms and plot the averaged cumulative regret over the time horizon
in Figure~\ref{fig:sim_main}. The structure-agnostic baseline incurs the largest regret throughout.
Incorporating network structure significantly improves performance:
the no-information method (NETC) reduces regret relative to the baseline,
and the two structure-aware algorithms (NSE and NSE-FS) further improve performance, achieving the lowest cumulative regret---the curves for partial and full support knowledge nearly coincide by the end of the horizon, consistent with the fact that both achieve $\sqrt{T}$-type regret under sparsity. The gap between the structure-aware methods and the baseline grows over time, indicating clear differences in learning efficiency.

\begin{figure}[ht]
\centering
\includegraphics[width=0.7\textwidth]{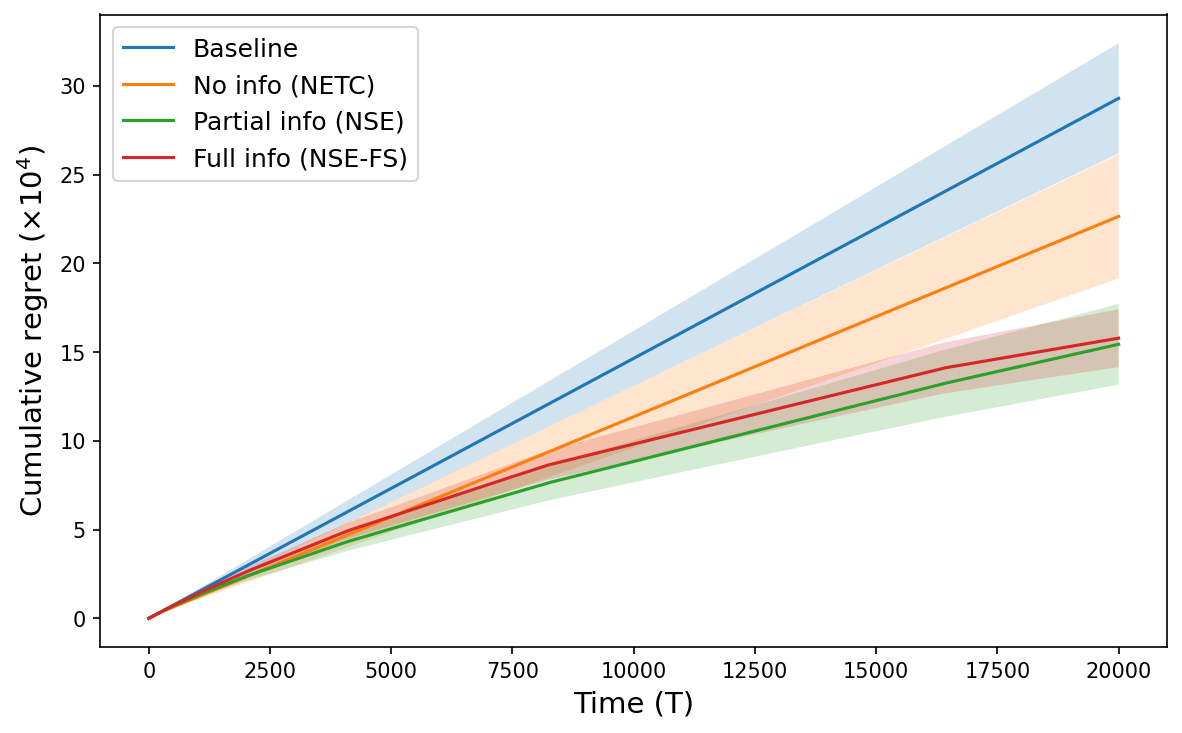}
\caption{Cumulative regret versus time of the four candidate algorithms. 
Blue denotes the baseline; 
orange denotes no network information (Algorithm~\ref{alg:netc}); 
green denotes partial network information (Algorithm~\ref{alg:col-support}); 
red denotes full network information (Algorithm~\ref{alg:full-support}). 
Solid lines represent the mean cumulative regret over 200 independent runs, 
and shaded regions indicate one standard deviation.}
\label{fig:sim_main}
\end{figure}

Next, we examine how regret scales with the network dimension $d$.
Figure~\ref{fig:per-individual} reports the per-individual cumulative regret, 
$\regret_{X^*}(\pi)/d$, of our three proposed algorithms as $d$ ranges from 100 to 900, 
with $s_0=20$ and $T=20000$ held fixed. The NETC algorithm, 
which uses no network information beyond the row-sparsity level, 
exhibits per-individual regret that is essentially flat in $d$, 
consistent with its $\tilde{O}((sT)^{2/3})$ per-individual bound. 
The two support-aware algorithms (NSE and NSE-FS) incur lower per-individual 
regret at small $d$ and grow only slowly as $d$ increases before plateauing. 
Overall, all three proposed algorithms maintain per-individual regret that is 
bounded in $d$ (up to log factors), in sharp contrast to the super-linear scaling of the baseline.

\begin{figure}[ht!]
\centering
\includegraphics[width=0.7\textwidth]{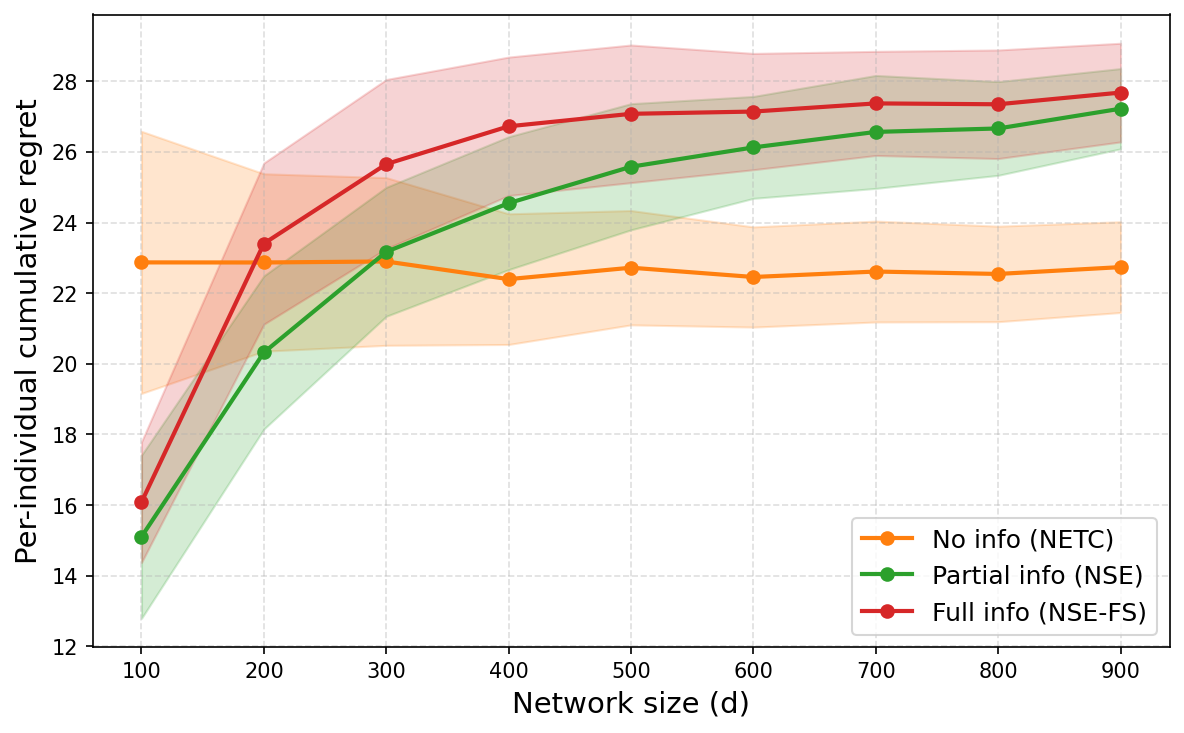}
\caption{Per-individual regret of proposed algorithms versus the network dimension $d$. 
Orange denotes no network information (Algorithm~\ref{alg:netc}); 
green denotes partial network information (Algorithm~\ref{alg:col-support}); 
red denotes full network information (Algorithm~\ref{alg:full-support}).
Solid lines represent the mean per-individual cumulative regret over 100 independent runs,
and shaded regions indicate one standard deviation.}
\label{fig:per-individual}
\end{figure}

\subsection{Semi-synthetic experiments with Indian village data}

Next, we use networks from the Indian village dataset published in \citet{banerjee2013diffusion}, which contains social networks of 74 villages in rural southern Karnataka, a state in India. The adjacency matrices in the dataset encode undirected graphs representing the social network of a village at the household level. Table~\ref{tab:village_summary_statistics} shows summary statistics of these networks, where sparsity is defined as the upper bound of the fractional row sparsity of each network, i.e., the largest number of other households to whom any single household is connected, divided by the total number of households. 
We use these networks as our sparse network structure, and then generate the random 
heterogeneous treatment effects for each edge in the network according to the model in~\eqref{eq:signal_function}.
With this fixed treatment effects matrix for each village, 
we repeat each of our algorithms 5 times using different random seeds. 
We compare the performance of all three proposed algorithms along with the baseline as before.

Since each village has a different number of individuals ($d$), we compare the per-individual cumulative regret of our algorithms, averaged across all 74 villages with 5 runs each. Figure~\ref{fig:village} shows the aggregated per-individual cumulative regret for $T = 20000$. As can be seen, the performance is consistent with the results from simulated data.

\begin{table}[t]
\centering
\caption{Summary statistics for Indian village dataset.}
\label{tab:village_summary_statistics}
\begin{tabular}{lrrrrr}
\toprule
Statistic & Mean & Median & Std & Min & Max \\
\midrule
Individuals ($d$) & 199.31 & 191.00 & 59.87 & 77.00 & 356.00 \\
Sparsity ($s/d$) & 0.20 & 0.18 & 0.07 & 0.10 & 0.37 \\
\bottomrule
\end{tabular}
\end{table}

\begin{figure}[t]
\centering
\includegraphics[width=0.6\textwidth]{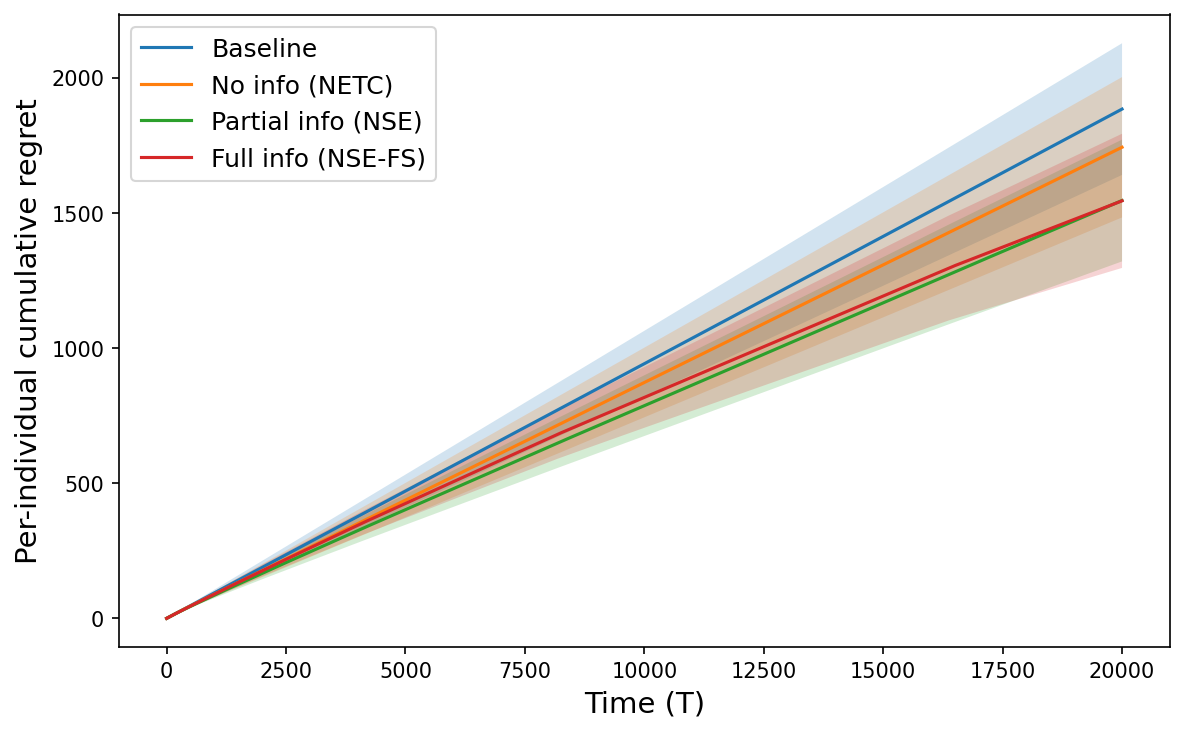}
\caption{Per-individual cumulative regret comparison. Blue denotes the baseline; 
orange denotes no network information (Algorithm~\ref{alg:netc}); 
green denotes partial network information (Algorithm~\ref{alg:col-support}); 
red denotes full network information (Algorithm~\ref{alg:full-support}). 
Solid lines represent the mean per-individual cumulative regret over 74 villages, each with 5 runs over different random seeds, 
and shaded regions indicate one standard deviation.}\label{fig:village}
\end{figure}

\section{Conclusion}

This paper studies adaptive targeting under network interference in a bandit setting 
where treatment effects may propagate across individuals through an unknown sparse network. 
Our work provides a fine-grained analysis of how the availability 
of structural information about the interference network affects the difficulty of online learning.
Our results highlight the importance of incorporating network structure when learning targeting policies in the presence of spillovers. 
Even limited knowledge of the interference structure can dramatically improve learning efficiency,
while fully ignoring network effects leads to substantial regret penalties. We close the paper by discussing the limitations and future directions.

\paragraph{Known sparsity level.} 
In practice, the sparsity parameter $s$ needs to be specified for NETC. Typically, we need to estimate $s$ either from prior knowledge about the network, by running a pilot study, or by using initial exploration data to estimate an upper bound on $s$ since it is usually unknown. We note that $s$ is used to determine the exploration length $T_1$ for Algorithm~\ref{alg:netc}. In practice, one can also use cross-validation to determine the length 
(see discussion in \citet{agarwal_multi-armed_2024}).

\paragraph{Future directions.} Several directions remain for future work. First, it would be interesting to close the remaining gap between upper and lower bounds in the unknown-network regime and determine whether the $(Ts)^{2/3}$ rate is optimal. Second, extending our framework to more general interference models---such as nonlinear spillovers or dynamic networks---would broaden its applicability. Finally, integrating inference objectives with regret minimization, as studied in recent work on experimentation under interference, is an important direction for designing practical adaptive experimentation systems.

\subsection*{Reproducibility}
Code for reproducing the numerical results in this paper is available at \url{https://github.com/Xiaomengwang99/Interference_bandit_public.git}.

\subsection*{Acknowledgments}
The authors would like to thank the Wharton Research Computing team for the computational resources
provided and the great support from the staff members.  
Z.~R.~is supported by the National Science Foundation (NSF) under grant DMS-2413135
and Wharton Analytics. 
\bibliographystyle{apalike}
\bibliography{main_reference}

@article{athey2016recursive,
  title={Recursive partitioning for heterogeneous causal effects},
  author={Athey, Susan and Imbens, Guido},
  journal={Proceedings of the National Academy of Sciences},
  volume={113},
  number={27},
  pages={7353--7360},
  year={2016},
  publisher={National Academy of Sciences}
}

@article{hao2020high,
  title={High-dimensional sparse linear bandits},
  author={Hao, Botao and Lattimore, Tor and Wang, Mengdi},
  journal={Advances in Neural Information Processing Systems},
  volume={33},
  pages={10753--10763},
  year={2020}
}

@book{lattimore2020bandit,
  title={Bandit algorithms},
  author={Lattimore, Tor and Szepesv{\'a}ri, Csaba},
  year={2020},
  publisher={Cambridge University Press}
}

@article{abbasi2011improved,
  title={Improved algorithms for linear stochastic bandits},
  author={Abbasi-Yadkori, Yasin and P{\'a}l, D{\'a}vid and Szepesv{\'a}ri, Csaba},
  journal={Advances in neural information processing systems},
  volume={24},
  year={2011}
}

@inproceedings{abbasi2012online,
  title={Online-to-confidence-set conversions and application to sparse stochastic bandits},
  author={Abbasi-Yadkori, Yasin and P{\'a}l, D{\'a}vid and Szepesv{\'a}ri, Csaba},
  booktitle={Artificial Intelligence and Statistics},
  pages={1--9},
  year={2012},
  organization={PMLR}
}

@inproceedings{carpentier2012bandit,
  title={Bandit theory meets compressed sensing for high dimensional stochastic linear bandit},
  author={Carpentier, Alexandra and Munos, R{\'e}mi},
  booktitle={Artificial Intelligence and Statistics},
  pages={190--198},
  year={2012},
  organization={PMLR}
}

@inproceedings{wang2018minimax,
  title={Minimax concave penalized multi-armed bandit model with high-dimensional covariates},
  author={Wang, Xue and Wei, Mingcheng and Yao, Tao},
  booktitle={International Conference on Machine Learning},
  pages={5200--5208},
  year={2018},
  organization={PMLR}
}

@article{kim2019doubly,
  title={Doubly-robust lasso bandit},
  author={Kim, Gi-Soo and Paik, Myunghee Cho},
  journal={Advances in Neural Information Processing Systems},
  volume={32},
  year={2019}
}

@article{bastani2020online,
  title={Online decision making with high-dimensional covariates},
  author={Bastani, Hamsa and Bayati, Mohsen},
  journal={Operations Research},
  volume={68},
  number={1},
  pages={276--294},
  year={2020},
  publisher={INFORMS}
}

@article{lattimore2015linear,
  title={Linear multi-resource allocation with semi-bandit feedback},
  author={Lattimore, Tor and Crammer, Koby and Szepesv{\'a}ri, Csaba},
  journal={Advances in Neural Information Processing Systems},
  volume={28},
  year={2015}
}

@article{ren2023dynamic,
  title={Dynamic batch learning in high-dimensional sparse linear contextual bandits},
  author={Ren, Zhimei and Zhou, Zhengyuan},
  journal={Management Science},
  year={2023},
  publisher={INFORMS}
}

@inproceedings{li2022interference,
  title={Interference, bias, and variance in two-sided marketplace experimentation: Guidance for platforms},
  author={Li, Hannah and Zhao, Geng and Johari, Ramesh and Weintraub, Gabriel Y},
  booktitle={Proceedings of the ACM Web Conference 2022},
  pages={182--192},
  year={2022}
}

@article{johari2022experimental,
  title={Experimental design in two-sided platforms: An analysis of bias},
  author={Johari, Ramesh and Li, Hannah and Liskovich, Inessa and Weintraub, Gabriel Y},
  journal={Management Science},
  volume={68},
  number={10},
  pages={7069--7089},
  year={2022},
  publisher={INFORMS}
}

@inproceedings{ugander2013graph,
  title={Graph cluster randomization: Network exposure to multiple universes},
  author={Ugander, Johan and Karrer, Brian and Backstrom, Lars and Kleinberg, Jon},
  booktitle={Proceedings of the 19th ACM SIGKDD international conference on Knowledge discovery and data mining},
  pages={329--337},
  year={2013}
}

@article{banerjee2013diffusion,
  title={The diffusion of microfinance},
  author={Banerjee, Abhijit and Chandrasekhar, Arun G and Duflo, Esther and Jackson, Matthew O},
  journal={Science},
  volume={341},
  number={6144},
  pages={1236498},
  year={2013},
  publisher={American Association for the Advancement of Science}
}

@article{sussman2017elements,
  title={Elements of estimation theory for causal effects in the presence of network interference},
  author={Sussman, Daniel L and Airoldi, Edoardo M},
  journal={arXiv preprint arXiv:1702.03578},
  year={2017}
}

@article{akbarpour2020just,
  title={Just a few seeds more: value of network information for diffusion},
  author={Akbarpour, Mohammad and Malladi, Suraj and Saberi, Amin},
  journal={Available at SSRN 3062830},
  year={2020}
}

@inproceedings{toulis2013estimation,
  title={Estimation of causal peer influence effects},
  author={Toulis, Panos and Kao, Edward},
  booktitle={International conference on machine learning},
  pages={1489--1497},
  year={2013},
  organization={PMLR}
}

@article{yu2022estimating,
  title={Estimating the total treatment effect in randomized experiments with unknown network structure},
  author={Yu, Christina Lee and Airoldi, Edoardo M and Borgs, Christian and Chayes, Jennifer T},
  journal={Proceedings of the National Academy of Sciences},
  volume={119},
  number={44},
  pages={e2208975119},
  year={2022},
  publisher={National Academy of Sciences}
}

@article{bojinov2023design,
  title={Design and analysis of switchback experiments},
  author={Bojinov, Iavor and Simchi-Levi, David and Zhao, Jinglong},
  journal={Management Science},
  volume={69},
  number={7},
  pages={3759--3777},
  year={2023},
  publisher={INFORMS}
}

@article{jia2024multiarmed,
  title={Multi-armed bandits with interference},
  author={Jia, Su and Frazier, Peter and Kallus, Nathan},
  journal={arXiv preprint arXiv:2402.01845},
  year={2024}
}

@article{leung2020treatment,
  title={Treatment and spillover effects under network interference},
  author={Leung, Michael P},
  journal={Review of Economics and Statistics},
  volume={102},
  number={2},
  pages={368--380},
  year={2020},
  publisher={MIT Press One Rogers Street, Cambridge, MA 02142-1209, USA journals-info~…}
}

@article{eckles2017design,
  title={Design and analysis of experiments in networks: Reducing bias from interference},
  author={Eckles, Dean and Karrer, Brian and Ugander, Johan},
  journal={Journal of Causal Inference},
  volume={5},
  number={1},
  pages={20150021},
  year={2017},
  publisher={De Gruyter}
}

@article{hu2022average,
  title={Average direct and indirect causal effects under interference},
  author={Hu, Yuchen and Li, Shuangning and Wager, Stefan},
  journal={Biometrika},
  volume={109},
  number={4},
  pages={1165--1172},
  year={2022},
  publisher={Oxford University Press}
}

@article{xu2024linear,
  title={Linear Contextual Bandits with Interference},
  author={Xu, Yang and Lu, Wenbin and Song, Rui},
  journal={arXiv preprint arXiv:2409.15682},
  year={2024}
}

@article{shirani2023causal,
  title={Causal Message Passing: A Method for Experiments with Unknown and General Network Interference},
  author={Shirani, Sadegh and Bayati, Mohsen},
  journal={arXiv preprint arXiv:2311.08340},
  year={2023}
}

@inproceedings{seeman2013adaptive,
  title={Adaptive seeding in social networks},
  author={Seeman, Lior and Singer, Yaron},
  booktitle={2013 IEEE 54th Annual Symposium on Foundations of Computer Science},
  pages={459--468},
  year={2013},
  organization={IEEE}
}

@article{bastani2021efficient,
  title={Efficient and targeted COVID-19 border testing via reinforcement learning},
  author={Bastani, Hamsa and Drakopoulos, Kimon and Gupta, Vishal and Vlachogiannis, Ioannis and Hadjichristodoulou, Christos and Lagiou, Pagona and Magiorkinis, Gkikas and Paraskevis, Dimitrios and Tsiodras, Sotirios},
  journal={Nature},
  volume={599},
  number={7883},
  pages={108--113},
  year={2021},
  publisher={Nature Publishing Group UK London}
}

@inproceedings{li2010contextual,
  title={A contextual-bandit approach to personalized news article recommendation},
  author={Li, Lihong and Chu, Wei and Langford, John and Schapire, Robert E},
  booktitle={Proceedings of the 19th international conference on World wide web},
  pages={661--670},
  year={2010}
}

@article{agarwal_multi-armed_2024,
  title={Mutli-armed bandits with network interference},
  author={Agarwal, Abhineet and Agarwal, Anish and Masoero, Lorenzo and Whitehouse, Justin},
  journal={Advances in Neural Information Processing Systems},
  volume={37},
  pages={36414--36437},
  year={2024}
}

@article{jamshidi2025graph,
  title={Graph-Dependent Regret Bounds in Multi-Armed Bandits with Interference},
  author={Jamshidi, Fateme and Shahverdikondori, Mohammad and Kiyavash, Negar},
  journal={arXiv preprint arXiv:2503.07555},
  year={2025}
}

@article{li2022random,
  title={Random graph asymptotics for treatment effect estimation under network interference},
  author={Li, Shuangning and Wager, Stefan},
  journal={The Annals of Statistics},
  volume={50},
  number={4},
  pages={2334--2358},
  year={2022},
  publisher={Institute of Mathematical Statistics}
}

@article{wang2025designbasedbandits,
  title={Design-based bandits under network interference: Trade-off between regret and statistical inference},
  author={Wang, Zichen and Hong, Haoyang and Li, Chuanhao and Li, Haoxuan and Zhang, Zhiheng and Wang, Huazheng},
  journal={Advances in Neural Information Processing Systems},
  volume={38},
  pages={119314--119345},
  year={2026}
}

@article{gleich2025scalablepolicy,
  title={Scalable Policy Maximization Under Network Interference},
  author={Gleich, David and Laber, Eric and Volfovsky, Alexander},
  journal={arXiv preprint arXiv:2505.18118},
  year={2025},
  note={AISTATS 2026 poster}
}

@article{zhang2024onlineexpdesign,
  title={Online experimental design with estimation-regret trade-off under network interference},
  author={Zhang, Zhiheng and Wang, Zichen},
  journal={Advances in Neural Information Processing Systems},
  volume={38},
  pages={15910--15950},
  year={2026}
}

@article{viviano2025policytargeting,
  title={Policy Targeting under Network Interference},
  author={Viviano, Davide},
  journal={Review of Economic Studies},
  volume={92},
  number={2},
  pages={1257--1292},
  year={2025}
}

@article{schroder2026exposure,
  title={Causal Inference on Networks under Misspecified Exposure Mappings: A Partial Identification Framework},
  author={Schr{\"o}der, Lukas and Oprescu, Miruna and Feuerriegel, Stefan and Kallus, Nathan},
  journal={arXiv preprint arXiv:2602.03459},
  year={2026}
}

@article{huber2026learningexposure,
  title={Learning and Testing Exposure Mappings of Interference using Graph Convolutional Autoencoder},
  author={Huber, Matthias and Kueck, Julian and Mattes, Martin},
  journal={arXiv preprint arXiv:2601.05728},
  year={2026}
}

@article{aronow2017estimating,
 ISSN = {19326157},
 URL = {http://www.jstor.org/stable/26362172},
 author = {Peter M. Aronow and Cyrus Samii},
 journal = {The Annals of Applied Statistics},
 number = {4},
 pages = {1912--1947},
 publisher = {Institute of Mathematical Statistics},
 title = {ESTIMATING AVERAGE CAUSAL EFFECTS UNDER GENERAL INTERFERENCE, WITH APPLICATION TO A SOCIAL NETWORK EXPERIMENT},
 urldate = {2026-05-03},
 volume = {11},
 year = {2017}
}

@article{tropp2015introduction,
  title={An introduction to matrix concentration inequalities},
  author={Tropp, Joel A},
  journal={Foundations and trends{\textregistered} in machine learning},
  volume={8},
  number={1-2},
  pages={1--230},
  year={2015},
  publisher={Emerald Publishing Limited}
}

@article{basse2019randomization,
  title={Randomization tests of causal effects under interference},
  author={Basse, Guillaume W and Feller, Avi and Toulis, Panos},
  journal={Biometrika},
  volume={106},
  number={2},
  pages={487--494},
  year={2019},
  publisher={Oxford University Press}
}

\newpage
\appendix

\section{Proof of main theorems}

\subsection{Proof of Theorem~\ref{thm:instance-baselineLB}}\label{app:baselineLB}
Fix a policy $\pi$ such that $a_t \in \sigma(a_1,Z_1,\ldots,a_{t-1},Z_{t-1})$
for any $t \in [T]$. 
We consider a set of instances indexed by $\{-1,1\}^d$. To be specific,  
for any $\sigma \in \{-1,1\}^d$, we define the network treatment effect matrix $X^{(\sigma)}$
by letting $X^{(\sigma)}_{ij} = S_{ij}\Delta_j\sigma_j$, $\forall i,j \in [d]$, where 
$\Delta_1,\ldots,\Delta_d \ge 0$ are scaling constants to be determined later.
By construction, it can be checked that the support of $X^{(\sigma)}$ coincides with $S$, 
and the sign of each column $j$ is determined by $\sigma_j$.
For notational simplicity, we shall use $\EE_{\sigma}$ and 
$\PP_{\sigma}$ to refer to the expectation and 
probability under $X^{(\sigma)}$.

Note that the incurred reward of any action $a \in \mathcal{A}$ can be computed as
\$
\one^\top X^{(\sigma)} a = \sum^{d}_{i,j=1} S_{ij}\Delta_j \sigma_j a_j = 
\sum^d_{j=1} \rho_j \Delta_j \sigma_j a_j.
\$
We then lower bound the worst-case regret by 
the average regret over the instances indexed by $\{-1,1\}^d$:
\begin{align}\label{eq:lower_bnd_2}
\sup_{X \in \mathbb{M}(S)} \EE[\regret_X(\pi)] & \ge \frac{1}{2^d} \sum_{\sigma \in \{-1,1\}^d}
\EE[\regret_{X^{(\sigma)}}(\pi)] \notag\\
  & =  \frac{1}{2^d} \sum_{\sigma \in \{-1,1\}^d}
  \EE_{\sigma}\Bigg[\sum^{T}_{t=1}
  \sum^{d}_{j=1} \big(\rho_j \Delta_j \sigma_j a^*_j - \rho_j \Delta_j
  \sigma_j a_{t,j} \big)\Bigg] \nonumber \\
  & =  \frac{1}{2^d} \sum_{\sigma \in \{-1,1\}^d}
  \sum_{j=1}^{d} \rho_j \Delta_j 
  \EE_{\sigma}\Bigg[\sum^{T}_{t=1} \big(|\sigma_j| - \sigma_j a_{t,j}) \Bigg] \nonumber\\ 
  & \ge  \frac{1}{2^d} \sum_{\sigma \in \{-1,1\}^d}
  \sum_{j=1}^{d} \rho_j \Delta_j 
  \EE_{\sigma}\Bigg[\sum^{T}_{t=1} \one\{\sign(a_{t,j}) \neq \sigma_j\} \Bigg], 
\end{align}
where the last inequality is because when $\sign(a_{t,j}) \neq \sigma_j$, 
$|\sigma_j| - \sigma_j a_{t,j} = 1 - \sigma_j a_{t,j} \ge 1$.
Applying Markov's inequality, we further have that 
\[
 \eqref{eq:lower_bnd_2} \ge  \frac{T}{2^{d+1}} 
  \sum_{j=1}^{d} \sum_{\sigma \in \{-1,1\}^d} \rho_j \Delta_j 
  \PP_{\sigma}\Bigg(\sum^T_{t=1} \one\{\sign(a_{t,j}) \neq \sigma_j\} \ge 
  \frac{T}{2}\Bigg). 
\]
For any $\sigma \in \{-1,1\}^d$ and $j \in [d]$, we define 
$\sigma^{(j)} = (\sigma_1,\ldots,-\sigma_j,\ldots,\sigma_d)$, 
where we flip the sign of the $j$-th item of $\sigma$.
It can then be checked that $\one\{\sign(a_{t,j}) \neq \sigma_j \} = \one\{\sign(a_{t,j}) = \sigma_j^{(j)}\}$.
We then have 
\begin{align*}
&\PP_{\sigma}\bigg(\sum^T_{t=1} \one\{\sign(a_{t,j}) \neq \sigma_j\} \ge 
\frac{T}{2}\bigg)+ 
\PP_{\sigma^{(j)}}\bigg(\sum^T_{t=1} \one\{\sign(a_{t,j}) \neq \sigma_j^{(j)}\} \ge 
\frac{T}{2}\bigg) \\
=&\PP_{\sigma}\bigg(\sum^T_{t=1} \one\{\sign(a_{t,j}) \neq \sigma_j\} \ge 
\frac{T}{2}\bigg)+ 
\PP_{\sigma^{(j)}}\bigg(\sum^T_{t=1} \one\{\sign(a_{t,j}) \neq  \sigma_j\} \le 
\frac{T}{2}\bigg)\\ 
\ge &\PP_{\sigma}\bigg(\sum^T_{t=1} \one\{\sign(a_{t,j}) \neq \sigma_j\} \ge 
\frac{T}{2}\bigg)+ 1-  
\PP_{\sigma^{(j)}}\bigg(\sum^T_{t=1} \one\{\sign(a_{t,j}) \neq  \sigma_j\} \ge 
\frac{T}{2}\bigg)\\ 
\ge& 1 - \text{TV}(\PP_{\sigma}, \PP_{\sigma^{(j)}}) \\
\ge& 1 - \sqrt{\tfrac{1}{2}\cdot D_{\rm KL}(\PP_{\sigma} \,\|\, \PP_{\sigma^{(j)}})},
\end{align*}
where the penultimate step follows from the definition of TV distance, and
the last step follows from Pinsker's inequality.
Since the algorithm operates on $(a_t,Z_t)_{t=1}^T$, we 
can decompose the KL-divergence between $\PP_\sigma$ and 
$\PP_{\sigma^{(j)}}$ as 
\begin{align*}
D_{\rm KL}(\PP_{\sigma} \,\|\, \PP_{\sigma^{(j)}})
&=  \EE_{\sigma}\Bigg[\log\Bigg(\prod_{t=1}^T 
\frac{d\PP_{\sigma, Z_{t} \given a_t}}{d\PP_{\sigma^{(j)}, Z_{t} \given a_t}}
\Bigg)\Bigg]
=  \sum^T_{t=1}
\EE_{\sigma}\bigg[\log\Bigg(
\frac{d\PP_{\sigma, Z_{t} \given a_t}}{d\PP_{\sigma^{(j)}, Z_{t} \given a_t}}
\Bigg)\Bigg], 
\end{align*}
where $\PP_{\sigma}( Z_t \given a_t)$ denotes the distribution of $Z_t$ given $a_t$ under $X_{\sigma}$.
Since $\sigma$ differs from $\sigma^{(j)}$ in only one entry, the conditional means of $Z_t$
under the two instances differ by at most $2\rho_j \Delta_j |a_{t,j}|$.
Recalling that $\eta_t$ is $d$-sub-Gaussian and $|a_{t,j}| \le 1$, we have 
\begin{align*}
\sum^T_{t=1}
\EE_{\sigma}\Bigg[\log\Bigg(
\frac{d\PP_{\sigma, Z_{t} \given a_t}}{d\PP_{\sigma^{(j)}, Z_{t} \given a_t}}
\Bigg)\Bigg]
\le  \frac{2 T \rho_j^2 \Delta_j^2}{d}.
\end{align*}
We then take $\Delta_j = \min\{\frac{\sqrt{d}}{2\sqrt{T}\rho_j}, \frac{1}{s}\}$ if 
$\rho_j >0$, and $\Delta_j = 0$ otherwise. The choice ensures 
that for every row $i$ and action $a \in \cA$,
\[
\big|X^{(\sigma)}_{i\cdot} a\big|
\le \sum_{j:S_{ij}=1} \Delta_j |a_j|
\le \sum_{j:S_{ij}=1} \frac{1}{s}
\le 1,
\]
so Assumption~\ref{asst:bounded} is satisfied.
By this choice, we have
\[
\PP_{\sigma}\Bigg(\sum^T_{t=1} \one\{\sign(a_{t,j}) \neq \sigma_j\} \ge 
\frac{T}{2}\Bigg) + 
\PP_{\sigma^{(j)}}\Bigg(\sum^T_{t=1} \one\{\sign(a_{t,j}) \neq \sigma_j^{(j)}\} \ge 
\frac{T}{2}\Bigg)\ge 1/2,
\]
and consequently 
\$
\eqref{eq:lower_bnd_2}   \ge &\frac{T}{2^{d+1}}
\sum^{d}_{j=1}\rho_j\Delta_j 
\sum_{\sigma \in \{\pm 1\}^d, \sigma_j = 1}
\Bigg\{\PP_{\sigma}\Bigg(\sum^T_{t=1} \one\{\sign(a_{t,j}) \neq \sigma_j\} \ge 
\frac{T}{2}\Bigg) +\PP_{\sigma^{(j)}}\Bigg(\sum^T_{t=1} \one\{\sign(a_{t,j}) \neq \sigma_j^{(j)}\} \ge 
\frac{T}{2}\Bigg) \Bigg\}\\ 
\ge & \frac{T}{8} \sum^d_{j=1}\rho_j\Delta_j\\
\!= & \frac{T}{8} \sum_{j=1}^d \min\Big\{\frac{\sqrt{d}}{2\sqrt{T}}\ind\{\rho_j > 0\},
\frac{\rho_j}{s}\Big\}\\
\!= & \frac{\sqrt{T}}{16}\sum_{j=1}^d
\min \Big\{\sqrt{d}\cdot \ind\{\rho_j > 0\}, \frac{2\sqrt{T}}{s}\rho_j\Big\}. 
\$
The proof is therefore completed.

\subsection{Proof of Theorem~\ref{thm:LB}}\label{app:LB}
Fix a policy $\pi$ such that $a_t \in \sigma(a_1,Y_1,\ldots,a_{t-1},Y_{t-1})$, $\forall t \in [T]$. 
As in the proof of Theorem~\ref{thm:instance-baselineLB}, 
we consider a set of instances indexed by $\{-1,1\}^d$. 
For each $\sigma \in \{-1,1\}^d$,  we define the corresponding matrix to be
$X^{(\sigma)}_{ij} = S_{ij}\Delta_j\sigma_j$, $\forall i,j \in[d]$, 
where $\Delta_1,\ldots,\Delta_d > 0$ are scaling parameters to be determined shortly.
By construction, the support of $X^{(\sigma)}$ coincides with $S$.

Lower bounding the worst case regret by the average regret over the 
instances in $\{-1,1\}^d$, we have 
\begin{align}\label{eq:lower_bnd}
  \sup_{X \in \mathbb{M}(S)} ~ \EE[\regret_X(\pi)]
  \ge \frac{1}{2^d}\sum_{\sigma \in \{-1,1\}^d} 
  \EE[\regret_{X^{(\sigma)}}(\pi)]. 
\end{align}
For each $\sigma \in \{-1,1\}^d$, we can explicitly write out the regret and 
therefore 
\begin{align}\label{eq:lower_bnd_mid}
  \eqref{eq:lower_bnd}  & = \frac{1}{2^d} \sum_{\sigma \in \{-1,1\}^d}
  \EE_{\sigma}\Bigg[\sum^{T}_{t=1}
  \sum^{d}_{j=1}\Delta_j\rho_j\cdot (\sigma_j a^*_j - \sigma_j a_{t,j}) \Bigg] \nonumber \\
  & =  \frac{1}{2^d} \sum_{\sigma \in \{-1,1\}^d}
  \sum_{j=1}^{d} \Delta_j \rho_j 
  \EE_{\sigma}\Bigg[\sum^{T}_{t=1} (1 - \sigma_j a_{t,j}) \Bigg] \nonumber \\
  & \ge  \frac{1}{2^d} \sum_{\sigma \in \{-1,1\}^d}
  \sum_{j=1}^{d} \Delta_j\rho_j 
  \EE_{\sigma}\Bigg[\sum^{T}_{t=1} \one\{\sign(a_{t,j}) \neq \sigma_j\} \Bigg]  \nonumber \\
  & {\ge}  \frac{T}{2^{d+1}} 
  \sum_{\sigma \in \{-1,1\}^d}\sum^d_{j=1}\Delta_j\rho_j
  \PP_{\sigma}\bigg(\sum^T_{t=1} \one\{\sign(a_{t,j}) \neq \sigma_j\} \ge 
  \frac{T}{2}\bigg).
\end{align}
where the last step follows from Markov's inequality.
Given $j \in [d]$ and $\sigma \in \{-1,1\}^d$, let $\sigma^{(j)}$
be the vector obtained by flipping the sign of the $j$-th coordinate of $\sigma$. 
Then,
\begin{align*}
& \PP_{\sigma}\bigg(\sum^T_{t=1} \one\{\sign(a_{t,j}) \neq \sigma_j\} \ge 
\frac{T}{2}\bigg) + 
\PP_{\sigma^{(j)}}\bigg(\sum^T_{t=1} \one\{\sign(a_{t,j}) \neq \sigma_j^{(j)}\} \ge 
\frac{T}{2}\bigg) \\
\ge & \PP_{\sigma}\bigg(\sum^T_{t=1} \one\{\sign(a_{t,j}) \neq \sigma_j\} \ge 
\frac{T}{2}\bigg) + 
1 - \PP_{\sigma^{(j)}}\bigg(\sum^T_{t=1} \one\{\sign(a_{t,j}) \neq \sigma_j\} \ge 
\frac{T}{2}\bigg) \\
\ge & 1 - \text{TV}(\PP_{\sigma}, \PP_{\sigma^{(j)}})
\ge 1 - \sqrt{\tfrac{1}{2}\cdot D_{\rm KL}(\PP_{\sigma} \,\|\, \PP_{\sigma^{(j)}})},
\end{align*}
where the penultimate step uses the definition of the TV distance and
the last step follows from Pinsker's inequality.   
Let $P_{\sigma}(Y_{k,t} \given a_t)$ denote the distribution of 
$Y_{k,t}$ given $a_t$ under $X_{\sigma}$, $\forall k\in[d], t\in[T]$.  
We then have
\begin{align*}
D_{\rm KL}(\PP_{\sigma} \,\|\, \PP_{\sigma^{(j)}})
=&  \EE_{\sigma}\Bigg[\log\Bigg(\prod_{t=1}^T \prod^d_{k=1}
\frac{d\PP_{\sigma, Y_{k,t} \given a_t}}{d\PP_{\sigma^{(j)}, Y_{k,t} \given a_t}}
\Bigg)\Bigg]
= \sum^T_{t=1}\sum^d_{k=1}
\EE_{\sigma}\Bigg[\log\Bigg(
\frac{d\PP_{\sigma, Y_{k,t} \given a_t}}{d\PP_{\sigma^{(j)}, Y_{k,t} \given a_t}}
\Bigg)\Bigg]
\le 2T\rho_j\Delta_j^2,
\end{align*}
where the last step uses the sub-gaussianity of the noise terms $\epsilon_{k,t}$
and that $X_{\sigma}$ differs from $X_{\sigma^{(j)}}$ in only one column.
Taking $\Delta_j = \min\{\frac{1}{2\sqrt{\rho_jT}},\frac{1}{s}\}$, 
which ensures that $\sup_{a \in \cA}|X^{(\sigma)}_{i\cdot} a| \le 1$, 
we have 
\begin{align*}
  \eqref{eq:lower_bnd_mid} = &
  \frac{T}{2^{d+1}}
  \sum^{d}_{j=1}\Delta_j \rho_j \sum_{\sigma \in \{-1,1\}^d, \sigma_j=1} \Bigg\{
  \PP_{\sigma}\Bigg(\sum^T_{t=1} \one\{\sign(a_{t,j}) \neq \sigma_j\} \ge 
  \frac{T}{2}\Bigg)+ \PP_{\sigma^{(j)}}\Bigg(\sum^T_{t=1} \one\{\sign(a_{t,j}) \neq \sigma^{(j)}_j\} \ge 
  \frac{T}{2}\Bigg)\Bigg\}\\
  \ge & \frac{T}{8} \sum^d_{j=1}\rho_j \min\Big\{\frac{1}{2\sqrt{\rho_j T}},\frac{1}{s}\Big\} \\
  = & \frac{1}{16} \cdot \sum^d_{j=1} \min\Big\{\sqrt{T\rho_j}, \frac{2T\rho_j}{s}\Big\}
\end{align*}
The proof is therefore completed.

\subsection{Proof of Theorem~\ref{thm:full-support}}\label{app:full-support}
For coordinates with $\theta_j = 0$, the regret contribution is zero, so we restrict the following argument to coordinates with $\theta_j \neq 0$.
We first define two good events on which 
the undetermined estimators $\hat \theta_j^{(m)}$ are close to the true $\theta_j$,
for $m < m_0$ and $m \ge m_0$, respectively. 
\$ 
\cE_1 = \bigcap_{m=1}^{m_0-1} \bigcap_{j \in U_{m-1}} \Big\{\big|\hat \theta_j^{(m)} - \theta_j\big| \le
\rho_j\tau_m/2\Big\},~
\cE_2 = \bigcap_{m=m_0}^M \bigcap_{j \in U_{m-1}} \Big\{\big|\hat \theta_j^{(m)} - \theta_j\big| \le 
\sqrt{\rho_j}\tau_m/2\Big\},
\$
The following lemma shows that both $\cE_1$ and $\cE_2$ hold with high probability.
\begin{lemma}\label{lemma:full-support-highprob}
For any $\delta \in (0,1)$, with probability at least $1-\delta$, 
\$ 
&\big|\hat \theta_j^{(m)} - \theta_j\big| \le \rho_j \tau_m/2,~\forall m < m_0 \text{ and } j\in U_{m-1},\\
&\big|\hat \theta_j^{(m)} - \theta_j\big| \le \sqrt{\rho_j} \tau_m/2,~\forall m \ge m_0 \text{ and } j\in U_{m-1}.
\$
\end{lemma}

We defer the proof of Lemma~\ref{lemma:full-support-highprob} to 
Appendix~\ref{app:proof-full-support-highprob} and 
proceed supposing that $ \cE := \cE_1 \cap \cE_2$ holds. For each $j\in[d]$, define 
\$ 
\hat m_j = \max\big\{m \in [M]: j \in U_{m-1}\big\}, 
\$ 
which is the last batch on which the sign of $\theta_j$ is undetermined.

If $\hat m_j < m_0$, then since $\cE_1$ holds, there is
\$
|\theta_j - \hat \theta^{(\hat m_j)}_j| \le \rho_j\tau_{\hat m_j}/2.
\$
By the choice of the uncertainty set, there is
\$
|\hat \theta^{(\hat m_j)}_j| \ge \rho_j\tau_{\hat m_j} 
> |\hat \theta^{(\hat m_j)}_j - \theta_j| 
\Rightarrow \sign(\hat \theta_j^{(\hat m_j)}) = \sign(\theta_j).
\$ 
Similarly, if $m_0 \le \hat m_j < M$, then since $\cE_2$ holds, there is
\$
|\theta_j - \hat \theta^{(\hat m_j)}_j| \le \sqrt{\rho_j}\tau_{\hat m_j}/2,
\$  
Then by the choice of the uncertainty set, there is 
\$
|\hat \theta^{(\hat m_j)}_j| \ge \sqrt{\rho_j}\tau_{\hat m_j} 
> |\hat \theta^{(\hat m_j)}_j - \theta_j| 
\Rightarrow \sign(\hat \theta_j^{(\hat m_j)}) = \sign(\theta_j).
\$ 
In any case, the regret incurred by the actions in batches $m > \hat m_j$ is zero. We 
can therefore upper bound the regret as 
\@ \label{eq:fs-regret-1}
\regret_{X^*}(\pi_\nsefs)   
\le \sum_{j=1}^d \sum_{m=1}^{\hat m_j} \sum_{t \in D_m} \big\{|\theta_j| - a_{t,j}\theta_j\big\}.
\le 2\sum_{j=1}^d \sum_{m=1}^{\hat m_j} \sum_{t \in D_m}|\theta_j|.
\@
Meanwhile, when $1 < \hat m_j \le m_0$, then by the definition of $\hat m_j$ there is 
\$ 
\rho_j\tau_{\hat m_j-1} \ge 
|\hat \theta_j^{(\hat m_j-1)}| \ge 
|\theta_j| - |\hat \theta_j^{(\hat m_j-1)} - \theta_j| 
\ge |\theta_j| - \frac{\rho_j}{2}\tau_{\hat m_j-1},
\$
where the last inequality uses the definition of $\cE$. Rearranging the above gives us
\$
|\theta_j| \le \frac{3}{2}\rho_j\tau_{\hat m_j-1}
 = 12\rho_j\sqrt{\frac{\log(16d^2M/\delta)}{n_{\hat m_j-1}}}.
\$
The above inequality further implies that 
$n_{\hat m_j - 1} \le 144\rho_j^2 \log(16d^2M/\delta)/\theta_j^2$.
Similarly, when $\hat m_j > m_0$, there is
\$
\sqrt{\rho_j}\tau_{\hat m_j-1} \ge 
|\hat \theta_j^{(\hat m_j-1)}| \ge 
|\theta_j| - |\hat \theta_j^{(\hat m_j-1)} - \theta_j| 
\ge |\theta_j| - \frac{\sqrt{\rho_j}}{2}\tau_{\hat m_j-1},
\$
which gives us
\$
|\theta_j| \le \frac{3}{2}\sqrt{\rho_j}\tau_{\hat m_j-1} = 12\sqrt{\rho_j}\sqrt{\frac{\log(16d^2M/\delta)}{n_{\hat m_j-1}}}
\Rightarrow n_{\hat m_j - 1} \le 144\rho_j \log(16d^2M/\delta)/\theta_j^2.
\$
Putting the above two cases together, we have
\$ 
\eqref{eq:fs-regret-1} & \le  
2 \sum_{j=1}^d \ind\{\hat m_j = 1\} 2|\theta_j| + 
\ind\{1 < \hat m_j \le m_0\} 2^{\hat m_j +1}|\theta_j| + \ind\{\hat m_j > m_0\} 2^{\hat m_j + 1}|\theta_j|\\
& \le 2 \sum_{j=1}^d \ind\{\hat m_j = 1\} 2\rho_j +
\ind\{1< \hat m_j \le m_0\} \min \Bigg\{2^{m_0 +1}|\theta_j|, \frac{576\rho_j^2}{|\theta_j|}\log(16d^2M/\delta) \Bigg\}\\
& \qquad \qquad + 
\ind\{\hat m_j > m_0\} \min\Bigg\{T|\theta_j|, \frac{576 \rho_j}{|\theta_j|} \log(16d^2M/\delta) \Bigg\} \\ 
& \le 48 \sum^d_{j=1}2\rho_j+ 2^{m_0/2}\rho_j \sqrt{\log(8d^2M/\delta)} + \sqrt{\rho_j T\log(16d^2M/\delta)}\\
& = 48\sum^d_{j=1}2\rho_j + 36\rho_j \sqrt{s}\log(16d^2Ms/\delta)+ \sqrt{\rho_j T\log(16d^2M/\delta)},
\$
where the last step uses the definition of $m_0$.
By assumption, we have $\rho_js \le ds \le T$, $\forall j\in[d]$, which implies that 
the regret is upper bounded by $c_1\sum^d_{j=1}\sqrt{\rho_j T}\log(16d^2Ms/\delta)$,
for some absolute constant $c_1 > 0$.
Recalling that $\cE$ holds with probability at least $1-\delta$, we conclude the proof of Theorem~\ref{thm:full-support}. 

\subsection{Proof of Corollary~\ref{cor:full-support}}\label{app:full-support-expected}
Letting $\delta = 1/(Td)$, we invoke Theorem~\ref{thm:full-support} to obtain that with probability at least $1-1/(Td)$,
\begin{align*}
\regret_{X^*}(\pi_\nsefs) & \leq  c_1 \log\big(16d^3sT\log_2 T\big) \cdot  \sqrt{T}\sum^d_{j=1}  \sqrt{\rho_j}.
\end{align*}
Let us denote the above event as $\cG$. We then have
\$ 
\EE[\regret_{X^*}(\pi_\nsefs)] & = \EE[\regret_{X^*}(\pi_\nsefs) \ind\{\cG\}] + \EE[\regret_{X^*}(\pi_\nsefs) \ind\{\cG^c\}]\\
& \leq c_1 \log\big(16d^3sT\log_2 T\big) \cdot  \sqrt{T}\sum^d_{j=1}  \sqrt{\rho_j} + \frac{2Td}{Td} \\
& \le c_1 \log\big(16d^3sT\log_2 T\big) \cdot \sqrt{T}\sum^d_{j=1}  \sqrt{\rho_j}.
\$
The proof is complete.

\subsection{Proof of Theorem~\ref{thm:se_upper_bound}}
\label{app:se_upper_bound}
As before, we define the following good event
\$
\mathcal{E} = \bigcap_{m=1}^M \bigcap_{j \in U_{m-1}} 
\Big\{\big|\hat \theta_j^{(m)} - \theta_j \big| \le \rho_j\tau_m/2 \Big\}. 
\$
The following lemma shows that $\mathcal{E}$ holds with high probability.
\begin{lemma}
\label{lemma:good_event_whp}
Under the conditions in Theorem~\ref{thm:se_upper_bound},
for any $\delta \in (0,1)$,
the good event $\mathcal{E}$ holds with probability at least $1 - \delta$.
\end{lemma}
We provide its proof in Appendix~\ref{app:proof_good_event_whp} and 
proceed on the event $\mathcal{E}$.
Let $\hat m_j$ denote the last batch for which 
$j$ is still in the ``exploration'' phase, i.e., 
\$
\hat m_j := \max\big\{m \in [M]: j \in U_{m-1}\big\}.
\$
We decompose the regret as follows:
\@\label{eq:regret}
\regret_{X^*}(\pi) = \sum^M_{m=1} \sum_{t \in D_m}
\sum^d_{j=1} \big\{ |\theta_j| - \theta_j a_{t,j} \big\}.
\@
For any $j\in[d]$, $j \in U_{\hat m_j-1}$ by definition.
On the good event $\mathcal{E}$, at the end of batch $\hat m_j$, there is 
\$ 
\big|\hat \theta_j^{(\hat m_j)} - \theta_j \big| \le \rho_j \tau_{\hat m_j}/2
\$
Meanwhile, if $\hat m_j < M$, then by the choice of $\hat m_j$,
$|\hat \theta_j^{(\hat m_j)}| > \rho_j \tau_{\hat m_j}$, 
which implies that $\sign(\hat \theta_j^{(\hat m_j)}) = \sign(\theta_j)$.
In other words, for column $j$, no regret will be incurred when $t >  T_{\hat m_{j}}$, 
which is trivially true when $\hat m_j = M$.
Combining the above, we conclude that on $\mathcal{E}$,  
\$
\eqref{eq:regret} \le \sum^d_{j=1} 2|\theta_j|\cdot T_{\hat m_j}.
\$
On the other hand, on $\mathcal{E}$, for $j \in [d]$, if $\hat m_j > 1$, then by the definition of $\hat m_j$ there is 
\@\label{eq:regret2}
|\theta_j| \le \big|\theta_j - \hat \theta_j^{(\hat m_j- 1)}\big| + |\hat \theta_j^{(\hat m_j - 1)}| 
\le \rho_j\tau_{\hat m_j - 1}/2 + |\hat \theta_j^{(\hat m_j-1)}| 
\le \frac{3}{2}\rho_j\tau_{\hat m_j-1}, 
\@
where the first step uses the triangle inequality, the second step is by the 
definition of $\cE$, and the last step follows from the choice of $\hat m_j$.
Inverting the inequality in~\eqref{eq:regret2}, we have 
\$
n_{\hat m_j-1} 
\le  \frac{512 \rho_j^2 \log(4d^2\log_2(T)/\delta)}{\theta_j^2} 
\$ 
Furthermore, by construction, we have 
\$ 
T_{\hat m_j} = 2^{\hat m_j+1} - 2 \le 4n_{\hat m_j -1} 
\le \frac{2048 \rho_j^2 \log(4d^2\log_2(T)/\delta)}{\theta_j^2}. 
\$
Combining the above with the trivial bound $\regret_{X^*}(\pi) \le 2T\sum^d_{j=1} |\theta_j|$ gives us that,
\$ 
\regret_{X^*}(\pi_{\textnormal{se}})
& \le\sum_{j=1}^d 2|\theta_j| \ind\{\hat m_j=1\} + 
2|\theta_j| \cdot \min\bigg\{\frac{2048\rho_j^2 \log(4d^2\log_2(T)/\delta)}{\theta_j^2},T\bigg\}\\
& \le \sum^d_{j=1}2 \rho_j + \min\bigg\{\frac{4096 \rho_j^2 \log(4d^2\log_2(T)/\delta)}{|\theta_j|}, 2|\theta_j|T \bigg\}\\
& \le \sum^d_{j=1} 66\rho_j \sqrt{T\log(4d^2\log_2(T)/\delta)}. 
\$
Recalling that $\mathcal{E}$ holds with probability at least $1-\delta$, we conclude the proof of Theorem~\ref{thm:se_upper_bound}.

\subsection{Proof of Corollary~\ref{cor:se_upper_bound}}
\label{app:proof_of_exp_upper_bnd}
Taking $\delta = \frac{1}{Td}$ in Theorem~\ref{thm:se_upper_bound}, we arrive at
\$
\EE\big[\regret_{X^*}(\pi_{\nse})\big] 
& = \EE\bigg[\regret_{X^*}(\pi_{\nse}) \cdot  \one\Big\{\regret_{X^*}(\pi_{\nse})\le \sum^d_{j=1} 64\rho_j \sqrt{T\log(4d^3T\log_2(T))} \Big\}\bigg] \\
& \qquad \qquad + \EE\bigg[\regret_{X^*}(\pi_{\nse}) \cdot  \one \Big\{\regret_{X^*}(\pi_{\nse})> \sum^d_{j=1} 64\rho_j \sqrt{T\log(4d^3T\log_2(T))} \Big\}\bigg]\\
& \le  \sum^d_{j=1} 64\rho_j \sqrt{T\log(4d^3T\log_2(T) )} + 2dT \cdot \frac{1}{dT}\\ 
&\le \sum^d_{j=1} 66\rho_j \sqrt{T\log(4d^3T\log_2(T))}.
\$

\subsection{Proof of Theorem~\ref{thm:dpUB}}\label{app:dpUB}
We first decompose the regret as
\begin{align*}
\regret_{X^*}(\pi_{\netc}) 
& = \sum_{t=1}^T  \theta^\top(a^* - a_t) \\
& = \sum_{t=1}^{T_1} \theta^\top(a^* - a_t) + \sum_{t=T_1+1}^{T} \theta^\top(a^* - \hat a) \\ 
& \leq 2T_1 d+ (T-T_1) \theta^\top(a^* - \hat a), 
\end{align*}
where the last step uses $\sup_{a \in \cA} |X^*_{i,\cdot}a| \le 1$ by Assumption~\ref{asst:bounded}.
For the second term, we have 
\begin{align*}
\theta^\top(a^* - \hat a)
& = (\theta - \hat{\theta})^\top (a^* - \hat a) + \hat{\theta}^\top (a^* - \hat a)  
\stepa{\le} (\theta - \hat{\theta})^\top (a^* - \hat a)  
\stepb{\le} \|\theta - \hat{\theta}\|_1 \|a^* - \hat a\|_\infty 
\le 2 \|\theta - \hat \theta\|_1,
\end{align*}
where step (a) follows from the optimality of $\hat a$ and step (b) uses H\"{o}lder's inequality.
By the definition of $\theta$ and $\hat \theta$, we have
\$ 
\|\theta - \hat \theta\|_1 = \sum^d_{j=1} \Big|\sum^d_{i=1} (X^*_{ij} - \hat X_{ij})\Big|
\le \sum_{i,j=1}^d |X^*_{ij} - \hat X_{ij}| = \sum_{i=1}^d \|X^*_{i\cdot} - \hat{X}_{i\cdot}\|_1.
\$
The remaining task is to bound $\|X^*_{i\cdot} - \hat{X}_{i\cdot}\|_1$ for each $i\in[d]$.
By the optimality of $\hat X_{i,\cdot}$, we have that 
\$ 
\frac{1}{2T_1} \sum_{t=1}^{T_1} \big(Y_{i,t} - a_t^\top \hat X_{i,\cdot}\big)^2 
+ \lambda \big\|\hat X_{i,\cdot}\big\|_1
\le \frac{1}{2T_1} \sum_{t=1}^{T_1} \big(Y_{i,t} - a_t^\top X^*_{i,\cdot}\big)^2 
+ \lambda \big\|X^*_{i,\cdot}\big\|_1.
\$
Rearranging the above inequality, we get
\$ 
\frac{1}{2T_1} \sum_{t=1}^{T_1} \big[(X^*_{i,\cdot} - \hat X_{i,\cdot}) a_t\big]^2
& \le \frac{1}{T_1} \sum_{t=1}^{T_1} \epsilon_{i,t} (\hat X_{i,\cdot} - X^*_{i,\cdot})a_t 
+ \lambda (\|X  ^*_{i,\cdot}\|_1 - \|\hat X_{i,\cdot}\|_1)\\
& \le \Big\|\frac{1}{T_1} \sum^{T_1}_{t=1} \epsilon_{i,t} a_t \Big\|_{\infty} \|\hat X_{i,\cdot} - X^*_{i,\cdot}\|_1
+ \lambda (\|X  ^*_{i,\cdot}\|_1 - \|\hat X_{i,\cdot}\|_1).
\$
The last step uses H\"{o}lder's inequality. Next, we bound the first term on the right-hand side.
For any $x > 0$, 
\@\label{eq:lasso_inftybound}
\PP\Bigg(\max_{j \in [d]} \Big|\frac{1}{T_1}\sum^{T_1}_{t=1} \epsilon_{i,t} a_{t,j} \Big| \geq x\Bigg)
\le \sum_{j=1}^d \PP\Bigg(\Big|\frac{1}{T_1}\sum^{T_1}_{t=1} \epsilon_{i,t} a_{t,j} \Big| \geq x\Bigg)
\le 2d\exp\Big(-\frac{T_1x^2}{2}\Big), 
\@
where the first step uses the union bound and the second step follows from Hoeffding's inequality.
As an implication of~\eqref{eq:lasso_inftybound} and a union bound over $i\in[d]$, with probability at least $1 - \delta$, 
there is 
\@ \label{eq:bounded_event}
\Big\|\frac{1}{T_1} \sum^{T_1}_{t=1} \epsilon_{i,t} a_t \Big\|_{\infty}
\le 2\sqrt{\frac{2\log(2d^2/\delta)}{T_1}} \le \frac{\lambda}{2},
\@
by the choice of $\lambda$ in Algorithm~\ref{alg:netc}.
Conditional on the event in~\eqref{eq:bounded_event} and 
writing $\hat \Delta_i := (\hat X_{i,\cdot} - X^*_{i,\cdot})^\top$, we have
\@ \label{eq:oracle_ineq}
\frac{1}{2T_1} \sum_{t=1}^{T_1} (a_t^\top \hat \Delta_i)^2
& \le \frac{\lambda}{2} \|\hat \Delta_i\|_1 + \lambda (\|X^*_{i,\cdot}\|_1 - \|\hat X_{i,\cdot}\|_1) \\
& = \frac{\lambda}{2} \big\|\hat \Delta_{i}[S_i]\big\|_1 + \frac{\lambda}{2} \big\|\hat \Delta_{i}[S_i^c]\big\|_1
+ \lambda (\|X^*_{i, \cdot}[S_i]\|_1 - \|\hat X_{i, \cdot}[S_i]\|_1 - \|\hat X_{i, \cdot}[S_i^c]\|_1) \\
& \le \frac{3\lambda}{2} \|\hat \Delta_{i}[S_i]\|_1 - \frac{\lambda}{2} \|\hat \Delta_{i}[S_i^c]\|_1,
\@
where the last step uses the triangle inequality.
As a result, we can conclude that 
\@\label{eq:S3}
3\|\hat \Delta_i[S_i]\|_1 \geq \|\hat \Delta_i[S_i^c]\|_1.
\@

We next lower bound the left-hand side of the inequality in~\eqref{eq:oracle_ineq}. 
Recall that the actions $a_t$'s are independent Rademacher random vectors. 
The following lemma establishes a restricted eigenvalue condition for the design matrix formed by   
the actions $a_t$'s. The proof of the lemma uses standard sparse regression techniques and is deferred to Appendix~\ref{app:proof_restricted_eigen}.
\begin{lemma}\label{lemma:restricted_eigen}
For any $\delta \in (0,1)$, suppose that 
$T_1 \ge  224^2(2s\log(d)+\log(2/\delta))$. Then with probability at least $1 - \delta$, 
for any $u \in \RR^d$ such that $\|u[S^c]\|_1 \le 3 \|u[S]\|_1$, for some $S \subseteq [d]$ with $|S| \le s$, 
there is 
\$ 
\frac{1}{T_1} \sum_{t=1}^{T_1} (a_t^\top u)^2 \ge \frac{1}{2} \|u[\tilde{S}]\|_2^2,
\$ 
where $\tilde S$ is the union of $S$ and the set of indices corresponding to the $s$ largest entries in $u[S^c]$ in absolute value.
\end{lemma}
By assumption, $T_1 \ge 224^2 (2s\log(d)+\log(2/\delta))$, 
and applying Lemma~\ref{lemma:restricted_eigen} to $\hat \Delta_i$ yields
\$ 
\frac{1}{2}\|\hat \Delta_i[\widetilde{S_i}]\|_2^2 \le \frac{1}{T_1} \sum_{t=1}^{T_1} (a_t^\top \hat \Delta_i)^2
\le \frac{3\lambda}{2} \|\hat \Delta_i[S_i]\|_1 \le \frac{3\lambda}{2} \sqrt{s} \|\hat \Delta_i[S_i]\|_2
\le \frac{3\lambda}{2} \sqrt{s} \|\hat \Delta_i[\widetilde{S_i}]\|_2,
\$
Dividing both sides by $\|\hat \Delta_i[\widetilde{S_i}]\|_2/2$, we have 
\$ 
\|\hat \Delta_i[\widetilde{S_i}]\|_2 \le 3\lambda \sqrt{s}.
\$
By the choice of $\widetilde{S_i}$, we have 
\$ 
\|\hat \Delta_i\|_2^2 \le \|\hat \Delta_i[\widetilde{S_i}]\|_2^2 + 
\|\hat \Delta_i[\widetilde{S_i}\backslash S_i]\|_1^2/s^2
\le \|\hat \Delta_i[\widetilde{S_i}]\|^2_2. 
\$
Furthermore, 
\$ 
\|\hat \Delta_i\|_1 \le 4 \|\hat \Delta_i(S_i)\|_1 \le 4\sqrt{s} \| \hat \Delta_i(S_i)\|_2
\le 4\sqrt{s} \|\hat \Delta_i\|_2 \le 12 s \lambda.
\$
Summing over $i\in[d]$, we have that with probability at least $1 - \delta$, 
\$ 
\|\hat \theta - \theta \|_1 \le \sum^d_{i=1} \|\hat \Delta_i\|_1 \le 12 s \lambda d.
\$
Plugging the above bound into the regret decomposition, we have that with probability at least $1 - \delta$,
\$ 
\regret_{X^*}(\pi_{\netc}) \le 2T_1 d + 24 s \lambda d (T - T_1) \le 200\sqrt{\log\Big(\frac{2d^2}{\delta}\Big)}\cdot T^{2/3}s^{2/3}d,
\$
where the last step uses the choice of $T_1$ and $\lambda$ in Algorithm~\ref{alg:netc}.

\subsection{Proof of Corollary~\ref{cor:dpUB}}
\label{app:cor-dpUB}
Applying Theorem~\ref{thm:dpUB} with $\delta = 1/(Td)$, we have that 
\$ 
\EE\big[\regret_{X^*}(\pi_{\netc})\big] & = 
\EE\big[\regret_{X^*}(\pi_{\netc})\ind\{\regret_{X^*}(\pi_{\netc}) \le c_5 d (sT)^{2/3}\sqrt{2\log(2d^2T)}\}\big] \\
& \qquad+ \EE\big[\regret_{X^*}(\pi_{\netc})\ind\{\regret_{X^*}(\pi_{\netc}) > c_5 d (sT)^{2/3}\sqrt{2\log(2d^2T)}\}\big] \\
& \le c_5 d (sT)^{2/3}\sqrt{2\log(2d^2T)} + 2 \le c_6d(sT)^{2/3}\sqrt{\log(2d^2T)}.
\$

\section{Proofs of supporting lemmas}
\subsection{Proof of Lemma~\ref{lemma:full-support-highprob}}
\label{app:proof-full-support-highprob}
We consider the case of $m < m_0$ and $m \ge m_0$ separately. 
\paragraph{Warm-up phase.}
Fix $m<m_0$. For any $i\in[d]$ and $j\in U_{m-1}$, 
\$ 
\hat X_{ij}^{(m)} - X^*_{ij} 
& = \frac{1}{n_m}\sum_{t\in D_m} a_{t,j} y_{t,i} - X^*_{ij} \\
& = \frac{1}{n_m}\sum_{t\in D_m} a_{t,j} \bigg(\sum^d_{k=1}X^*_{ik} a_{t,k} + \epsilon_{t,i}\bigg) 
- X^*_{ij}\\ 
& = \frac{1}{n_m}\sum_{t\in D_m} \bigg(X_{ij}^* +\sum_{k \neq j}X_{ik}^*a_{t,k} a_{t,j} + \epsilon_{t,i} a_{t,j} \bigg) - X^*_{ij}\\
& = \frac{1}{n_m}\sum_{t\in D_m} \bigg(\sum_{k \neq j}X_{ik}^*a_{t,k} a_{t,j} + \epsilon_{t,i} a_{t,j}\bigg)
\$
Summing over $i$ for which $j \in U_{m-1}$, we have
\$ 
\hat \theta_j^{(m)} - \theta_j
& = \sum_{i=1}^d \frac{\ind\{j\in S_i\}}{n_m}\sum_{t\in D_m} \bigg(\sum_{k \neq j}X_{ik}^*a_{t,k} a_{t,j} + \epsilon_{t,i} a_{t,j}\bigg)\\
& = \frac{1}{n_m}\sum_{t \in D_m} \sum^d_{i=1} \ind\{j\in S_i\}\bigg(\sum_{k \neq j}X_{ik}^*a_{t,k} a_{t,j} + \epsilon_{t,i} a_{t,j}\bigg).
\$
Since $a_{t,j}$'s are independent Rademacher random variables for $t\in D_m$, 
$\EE[\hat\theta_j^{(m)} - \theta_j] = 0$. 
By Assumption~\ref{asst:bounded}, 
\$ 
\Big|\sum_{i=1}^d \ind\{j\in S_i\}\sum_{k \neq j}X_{ik}^*a_{t,k} a_{t,j}\Big| 
\le\sum_{i=1}^d \ind\{j\in S_i\} \Big| \sum_{k \neq j}X_{ik}^*a_{t,k} \Big|  
\le \rho_j \sup_{a \in \cA} |X^*_{i\cdot} a|\le \rho_j.
\$
By Hoeffding's inequality, we have for any $\eta >0$ that 
\$ 
& \PP\Bigg(\bigg|\frac{1}{n_m}\sum_{t\in D_m}\sum_{i=1}^d \ind\{j \in S_i\}\bigg(\sum_{k \neq j}X_{ik}^*a_{t,k} a_{t,j} + \epsilon_{t,i} a_{t,j}\bigg) \bigg|\ge 2\eta \Bigg)\\
\le~& \PP\Bigg(\bigg|\frac{1}{n_m}\sum_{t\in D_m} \sum_{i=1}^d \ind\{j \in S_i\} \sum_{k \neq j}X_{ik}^*a_{t,k} a_{t,j} \bigg| \ge \eta \Bigg)
+ \PP\Bigg(\bigg|\frac{1}{n_m}\sum_{t\in D_m} \sum_{i=1}^d \ind\{j \in S_i\} \epsilon_{t,i} a_{t,j} \bigg| \ge \eta \Bigg)\\
\le~& 2\exp\bigg(-\frac{\eta^2 n_m}{2 \rho_j^2}\bigg) + 2\exp\bigg(-\frac{\eta^2 n_m}{2 \rho_j}\bigg)
\le 4\exp\bigg(-\frac{\eta^2 n_m}{2 \rho_j^2}\bigg).
\$ 
Taking a union bound over $m < m_0$  and $j\in U_{m-1}$, and letting 
$\eta =  \sqrt{\rho_j \frac{2\log(8Md^2/\delta)}{n_m}}$, we obtain that 
with probability at least $1-\delta/2$, for any $m\in[M]$ and $j\in U_{m-1}$, 
\$
\big|\hat \theta_j^{(m)} - \theta_j\big| \le 2 \rho_j \sqrt{\frac{2\log(8Md^2/\delta)}{n_m}} = \rho_j \tau_m/2.
\$

\paragraph{OLS phase.}
Fix $m \ge m_0$. To ease presentation, we abuse notation and let $n_{m_0} = T_{m_0}$.
For each row $i \in [d]$, we assemble the actions in batch $m$ 
restricted to the indices in $E_i^{(m)}$ into a $n_m \times |E_i^{(m)}|$ matrix:
\$ 
\sfA_i^{(m)} = (a_{T_{m-1}+1}[E^{(m)}_i], \ldots, a_{T_{m}}[E^{(m)}_i])^\top,
\$
and its centered version
\$ 
\tilde \sfA_i^{(m)} = H\sfA := \Big(I - \frac{1}{n_m}\one \one^\top\Big) \sfA_i^{(m)}.
\$
Also let $\sfY_{i}^{(m)} = (Y_{t,i})_{t \in D_m} \in \RR^{n_m}$ denote the reward vector for unit $i$ in batch $m$. 
By construction, we have 
\@\label{eq:row-est}
\hat X_{i,E_i^{(m)}} = (\tilde \sfA_{i}^{(m)\top} \tilde \sfA_{i}^{(m)})^{-1} \tilde \sfA_{i}^{(m)\top} 
(\sfY_{i}^{(m)} - \bar Y_{i}^{(m)} \one),
\@ 
Note that for any $t\in D_m$, $a_{t,j} \equiv \sign(\hat \theta_j^{(m-1)})$ if $j \notin U_{m-1}$. Therefore, 
\$ 
Y_{t,i} - \bar Y_{i}^{(m)} &= X_{i,\cdot}^* a_t + \epsilon_{t,i} - \frac{1}{n_m}\sum_{s \in D_m} (X_{i,\cdot}^* a_s + \epsilon_{s,i})\\
& = X_{i,E_i^{(m)}}^* a_{t}[E_i^{(m)}] + \epsilon_{t,i} - 
\frac{1}{n_m}\sum_{s \in D_m} \big(X_{i,E_i^{(m)}}^* a_{s}[E_i^{(m)}] + \epsilon_{s,i} \big). 
\$
Using the above, we have 
\$ 
\eqref{eq:row-est} & = 
(\tilde \sfA_{i}^{(m)\top} \tilde \sfA_{i}^{(m)})^{-1} \tilde \sfA_{i}^{(m)\top} 
(\tilde \sfA_{i}^{(m)} X^*_{i,E_i^{(m)}} + \sfE_{i}^{(m)} - \bar \sfE_{i}^{(m)} \one_{n_m})\\
& = X^*_{i,E_i^{(m)}} \underbrace{ + (\tilde \sfA_{i}^{(m)\top} \tilde \sfA_{i}^{(m)})^{-1} \tilde \sfA_{i}^{(m)\top} ( \sfE_{i}^{(m)} - \bar \sfE_{i}^{(m)} \one_{n_m})}_{=:r_i^{(m)}},
\$
where $\sfE_{i}^{(m)} = (\epsilon_{t,i})_{t \in D_m}$ and $\bar \sfE_{i}^{(m)} = \frac{1}{n_m}\sum_{t \in D_m} \epsilon_{t,i}$ is the sample mean of the noise for unit $i$ in batch $m$. 
For any $j \in U_{m-1}$, we can now decompose the estimation error for $\theta_j$ as follows:
\$
\hat \theta_j^{(m)} - \theta_j = \sum_{i=1}^d r_{i, \pi_i(j)}^{(m)} \ind\{j \in S_i\},
\$
where $\pi_i(j)$ is the position of index $j$ in the set $E_i^{(m)}$. 

Next, we will bound the estimation error $r_i^{(m)}$.
To this end, we first show that with high probability, the minimum eigenvalue of $\tilde \sfA_{i}^{(m)\top} \tilde \sfA_{i}^{(m)}$ is bounded away from zero.

\paragraph{Eigenvalue lower bound of $\tilde \sfA_{i}^{(m)\top} \tilde \sfA_{i}^{(m)}$.}
For each $m\in[M]$,  we have by construction that  
\$
\tilde \sfA_i^{(m)\top} \tilde \sfA_i^{(m)}  = 
\sfA_i^{(m)\top} \sfA_i^{(m)} - \frac{1}{n_m} \sfA_i^{(m)\top} \one \one^\top \sfA_i^{(m)}.
\$
The triangular inequality gives that 
\$ 
\|\tilde \sfA_i^{(m)\top} \tilde \sfA_i^{(m)} - n_m I\| 
\le \|\sfA_i^{(m)\top} \sfA_i^{(m)} - n_m I\| + \Big\|\frac{1}{n_m} \sfA_i^{(m)\top} \one \one^\top \sfA_i^{(m)} \Big\|.
\$
We bound the two terms separately conditional on the previous batches, 
where we do not explicitly write the conditioning for ease of presentation. 

First, note that $a_{t,j} \stackrel{\text{iid}}{\sim} \text{Uniform}(\{-1,1\})$, for any $t \in D_m$ and $j \in E_i^{(m)}$ and 
\$ 
&\sfA_i^{(m)\top} \sfA_i^{(m)} - n_m I= \sum_{t \in D_m} \Big(a_{t}[E_i^{(m)}] a_{t}[E_i^{(m)}]^\top - I\Big), 
\$
where for each $t \in D_m$, there is 
\$ 
& \big\|a_{t}[E_i^{(m)}] a_{t}[E_i^{(m)}]^\top - I\big\| \le |E_i^{(m)}| \le s, \\ 
& \Bigg\|\sum_{t \in D_m} \EE\Big[\big(a_t[E_i^{(m)}]a_t[E_i^{(m)}]^\top - I\big) \big(a_t[E_i^{(m)}]a_t[E_i^{(m)}]^\top - I\big)\Big]\Bigg\|
\le n_m s.
\$
Applying the matrix Bernstein inequality~\citep[Theorem 6.1.1]{tropp2015introduction}, we have that for any $x >0$, 
\$ 
\PP\Big(\|\sfA_i^{(m)\top} \sfA_i^{(m)} - n_m I\| \ge x\Big) \le 
2 s \exp\Big(\frac{-x^2/2}{n_m s + s x/3}\Big).
\$

As for the second term, we can write 
\$
\Big\|\frac{1}{n_m} \sfA_i^{(m)\top} \one \one^\top \sfA_i^{(m)} \Big\| 
= \frac{1}{n_m} \|\sfA_i^{(m)\top} \one \|_2^2 
= \frac{1}{n_m} \bigg\|\sum_{t \in D_m} a_t[E_i^{(m)}] \bigg\|_2^2
= \frac{1}{n_m} \sum_{\ell \in E_i^{(m)}} \Bigg(\sum_{t \in D_m} a_{t,\ell}\Bigg)^2
\$
Applying Hoeffding's inequality to each sum $\sum_{t \in D_m} a_{t,\ell}$ and a union bound over $\ell \in E_i^{(m)}$, 
we have for any $y > 0$,
\$ 
\PP\Bigg(\sum_{\ell \in E_i^{(m)}} \Bigg(\sum_{t \in D_m} a_{t,\ell}\Bigg)^2 \ge sy^2\Bigg) \le 
\sum_{\ell \in E_i^{(m)}} \PP\Bigg(\Bigg|\sum_{t \in D_m} a_{t,\ell}\Bigg| \ge y \Bigg)
\le 2 s\exp\Big(-\frac{y^2}{2n_m}\Big)
\$
With the above bounds and taking a union bound over $m\in[M]$ and $i\in[d]$, 
we have with probability at least $1-\delta/4$ that 
\@ \label{eq:eigen_event}
\| \tilde \sfA^{(m)\top}_i \tilde \sfA_i^{(m)} - n_m I \| 
\le 2\sqrt{2 n_m s \log(16Mds/\delta)} + s \log(16Mds/\delta), \quad \forall m\in[M], i\in[d].
\@
We proceed conditioning on the high-probability event in~\eqref{eq:eigen_event}.

By assumption, $m \ge m_0$ and therefore $n_{m} \ge 128s \log(16Mds/\delta)$ and we have 
\$ 
\lambda_{\min}(\tilde \sfA_i^{(m)\top} \tilde \sfA_i^{(m)})
& = \lambda_{\min}(\tilde \sfA_i^{(m)\top} \tilde \sfA_i^{(m)} - n_m I + n_mI)\\
& \stepa{\ge} n_m - \big\|\tilde \sfA_i^{(m)\top} \tilde \sfA_i^{(m)} - n_m I\big\|\\
& \ge n_m - 2\sqrt{2n_ms \log(16Mds/\delta)}- s \log(16Mds/\delta)\\
& \ge n_m/2, 
\$
where step (a) follows from Weyl's inequality.
As a result, $\|(\tilde \sfA_i^{(m)\top} \tilde \sfA_i^{(m)})^{-1}\| \le  \frac{2}{n_m}$.
Next, we write 
\$
\sfM^{(i,m)} := (\tilde \sfA_i^{(m)\top} \tilde \sfA_i^{(m)})^{-1} \tilde \sfA_i^{(m)\top},
\$ 
and use $\sfM^{(i,m)}_{\pi_i(j)}$ to denote the $\pi_i(j)$-th row of $\sfM^{(i,m)}$.
Note that 
\$ 
\|\sfM^{(i,m)}\sfM^{(i,m)\top}\| & = \|(\tilde \sfA_i^{(m)\top} \tilde \sfA_i^{(m)} )^{-1}\|
\le \frac{2}{n_m}.
\$
As a result, we have a bound on the row norms:
\$
\|\sfM^{(i,m)}_{\pi_i(j)}\|_2^2 \le \|\sfM^{(i,m)}\sfM^{(i,m)\top}\| \le \frac{2}{n_m}.
\$
Applying Hoeffding's inequality conditional on the actions, for any $x>0$, we have
\$ 
\PP\Bigg(\bigg|\sum_{i=1}^d \ind\{j \in S_i\}\sfM^{(i,m)}_{\pi_i(j)}\sfE_i \bigg| \ge x\Bigg) & \le 
2 \exp\Bigg(\frac{-x^2/2}{\sum_{i=1}^d \ind\{j \in S_i\}\|\sfM^{(i,m)}_{\pi_i(j)}\|_2^2}\Bigg)\\
& \le  
2 \exp\Bigg(\frac{-x^2/2}{\sum_{i=1}^d \ind\{j \in S_i\} \cdot 2/n_m}\Bigg) 
= 2 \exp\Bigg(\frac{-n_m x^2}{4 \rho_j}\Bigg).
\$
Similarly, for any $y > 0$, we have
\$ 
\PP\Bigg(\bigg|\sum_{i=1}^d \ind\{j \in S_i\}(\sfM^{(i,m)}_{\pi_i(j)}\one) \bar\sfE^{(m)}_i \bigg| \ge y\Bigg) & \le
2 \exp\Bigg(\frac{-y^2/2}{\sum_{i=1}^d \ind\{j \in S_i\}\|\sfM^{(i,m)}_{\pi_i(j)}\one\|_2^2 \cdot \frac{1}{n_m}} \Bigg)\\ 
& \le 2 \exp\Bigg(\frac{-n_m y^2/2}{2\rho_j}\Bigg).
\$  
Taking a union bound over $j \in U_{m-1}$ and $m\in[M]$, as well as the event in~\eqref{eq:eigen_event}, 
we have with probability at least $1-\delta/2$ that for any $m \ge m_0$
and $j \in U_{m-1}$,
\$ 
|\hat \theta_j - \theta| \le 4\sqrt{\frac{\rho_j \log(16 d M/\delta)}{n_m}} 
= \sqrt{\rho_j} \tau_m/2.
\$

\subsection{Proof of Lemma~\ref{lemma:good_event_whp}}
\label{app:proof_good_event_whp}
Fix $m\in[M]$. For any $i\in[d]$ and $j\in U_{m-1}$, 
\$ 
\hat X_{ij}^{(m)} - X^*_{ij} 
& = \frac{1}{n_m}\sum_{t\in D_m} a_{t,j} y_{t,i} - X^*_{ij} \\
& = \frac{1}{n_m}\sum_{t\in D_m} a_{t,j} \bigg(\sum^d_{k=1}X^*_{ik} a_{t,k} + \epsilon_{t,i}\bigg) 
- X^*_{ij}\\ 
& = \frac{1}{n_m}\sum_{t\in D_m} \bigg(X_{ij}^* +\sum_{k \neq j}X_{ik}^*a_{t,k} a_{t,j} + \epsilon_{t,i} a_{t,j} \bigg) - X^*_{ij}\\
& = \frac{1}{n_m}\sum_{t\in D_m} \bigg(\sum_{k \neq j}X_{ik}^*a_{t,k} a_{t,j} + \epsilon_{t,i} a_{t,j}\bigg)
\$
Since $a_{t,j}$'s are independent Rademacher random variables for $t\in D_m$, 
$\EE[\hat X_{ij}^{(m)} - X^*_{ij}] = 0$. 
By Assumption~\ref{asst:bounded}, 
\$ 
\Big|\sum_{k \neq j}X_{ik}^*a_{t,k} a_{t,j}\Big| \le \Big| \sum_{k \neq j}X_{ik}^*a_{t,k} \Big|  
\le \sup_{a \in \cA} |X^*_{i\cdot} a|\le 1.
\$
By Hoeffding's inequality, we have for any $\eta >0$ that 
\$ 
& \PP\Bigg(\bigg|\frac{1}{n_m}\sum_{t\in D_m} \bigg(\sum_{k \neq j}X_{ik}^*a_{t,k} a_{t,j} + \epsilon_{t,i} a_{t,j}\bigg) \bigg|\ge 2\eta \Bigg)\\
\le~& \PP\Bigg(\bigg|\frac{1}{n_m}\sum_{t\in D_m} \sum_{k \neq j}X_{ik}^*a_{t,k} a_{t,j} \bigg| \ge \eta \Bigg)
+ \PP\Bigg(\bigg|\frac{1}{n_m}\sum_{t\in D_m} \epsilon_{t,i} a_{t,j} \bigg| \ge \eta \Bigg)\\
\le~& 4\exp\bigg(-\frac{\eta^2 n_m}{2}\bigg).
\$ 
Taking a union bound over $m\in [M]$, $i\in[d]$ and $j\in U_{m-1}$, and letting 
$\eta = \tau_m /16 =  \sqrt{\frac{2\log(4Md^2/\delta)}{n_m}}$, we obtain that 
with probability at least $1-\delta$, for any $m\in[M]$, $i\in[d]$ and $j\in U_{m-1}$, 
\$
\Big|\hat X_{ij}^{(m)} - X_{ij}^*\Big| \le \frac{\tau_m}{8}.
\$
On the above event, for any $m\in[M]$ and $j\in U_{m-1}$, we have
\$  
\big|\theta_j - \hat \theta_j^{(m)}\big|
\le & \sum_{i=1}^d \big|\hat X_{ij}^{(m)}\ind\{|\hat X_{ij}^{(m)}| > \tau_m/8\} - X_{ij}^*\big|\\
\le &\sum_{i=1}^d \ind\{X_{ij}^* \neq 0, |\hat X_{ij}^{(m)}| \le \tau_m/8\}
\cdot \big|X_{ij}^*\big| + \ind\{X_{ij}^* \neq 0, |\hat X_{ij}^{(m)}| > \tau_m/8\}
\cdot \big|\hat X_{ij}^{(m)} - X_{ij}^*\big| \\
\le & \rho_j \tau_m/2.
\$
The proof is complete.

\subsection{Proof of Lemma~\ref{lemma:restricted_eigen}}
\label{app:proof_restricted_eigen}
We define the restricted eigenvalue constants for a matrix $A \in \RR^{d\times d}$ with order $m\le d$ as
\$ 
\begin{aligned}
\phi_{\max}(A,m) &:= \sup_{v \in \RR^d: \|v\|_0 \le m}
\frac{v^\top A v}{\|v\|_2^2}, \\
\phi_{\min}(A,m) &:= \inf_{v \in \RR^d: \|v\|_0 \le m}
\frac{v^\top A v}{\|v\|_2^2}.
\end{aligned}
\$
Let $v \in \RR^d$ denote an arbitrary unit vector with $\|v\|_0 \le m$. 
We can check that for any $y \in \RR$,  
\$
\EE[e^{y (a_t^\top v)}] = \prod_{j=1}^d \EE[e^{y a_{t,j} v_j}]
= \prod_{j=1}^d \cosh(y v_j) \le \prod_{j=1}^d e^{y^2 v_j^2/2} = e^{y^2/2},
\$
which implies that $a_t^\top v$ is a sub-gaussian random variable with variance proxy at most $1$.
Consequently, $(a_t^\top v)^2 - 1$ is a sub-exponential random variable with parameters $(4\sqrt{2}, 4)$.
Applying Bernstein's inequality, we have that for any $x > 0$,
\$ 
\PP\Bigg(\bigg|\frac{1}{T_1} \sum_{t=1}^{T_1} (a_t^\top v)^2 - 1\bigg| \ge x\Bigg)
\le 2\exp\bigg(-\min \Big\{\frac{x^2}{64},\frac{x}{8}\Big\} \cdot T_1\bigg).
\$
Let $\cN(\epsilon)$ denote an $\epsilon$-net of the $m$-dimensional unit ball, $\mathbb{S}^{m-1}$, 
for some $\epsilon \in (0,1/3)$.
By a standard covering argument, the cardinality of $\cN(\epsilon)$ can be upper bounded as 
$\left(1 + \frac{2}{\epsilon}\right)^m$. We further denote $\overline{\cN(\epsilon)}$ the 
set of vectors in $\RR^d$ whose support is in $\cN(\epsilon)$, i.e.,
\$
\overline{\cN(\epsilon)} := \{v \in \RR^d: v[\supp(v)] \in \cN(\epsilon)\}.
\$
It then follows that the cardinality of $\overline{\cN(\epsilon)}$ is upper bounded by
\$
{d \choose m} \cdot \left(1 + \frac{2}{\epsilon}\right)^m \le \left(\frac{ed}{m}\right)^m \cdot \left(1 + \frac{2}{\epsilon}\right)^m.
\$ 
Taking a union bound over $v \in \overline{\cN(\epsilon)}$,  
we have that with probability at least $1-2\exp(-(\frac{x^2}{64}\wedge \frac{x}{8})T_1 + 
m\log(ed/m) + m\log(1 + 2/\epsilon))$, 
\@ \label{eq:epsilonnet}
\bigg|\frac{1}{T_1} \sum_{t=1}^{T_1} (a_t^\top v)^2 - 1\bigg| \le x, ~\forall v \in \overline{\cN(\epsilon)}.
\@
Conditional on the event in~\eqref{eq:epsilonnet}, for an arbitrary vector 
$s$-sparse $r \in \RR^d$, there exists $v_0 \in \overline{\cN(\epsilon)}$ such that 
$\supp(v_0) = \supp(r)$ and $\|v_0 - r\|_2 \le \epsilon$. Writing 
$\sfA := \frac{1}{T_1} \sum^{T_1}_{t=1} a_t a_t^\top$, we have  
\$ 
r^\top \sfA r -1 = r^\top \sfA (r - v_0) + (r - v_0)^\top \sfA v_0 + v_0^\top \sfA v_0 - 1
\le 2 \phi_{\max}(\sfA, s) \|r - v_0\|_2 + x  \le 2\phi_{\max}(\sfA, m) \epsilon + x.
\$
Taking the maximum over $r$, we further have that 
\$ 
\phi_{\max}(\sfA, m) \le \frac{1+x}{1-2\epsilon} = 1 + \frac{2\epsilon + x}{1-2\epsilon}. 
\$
On the other hand, 
\$ 
r^\top \sfA r - 1 \ge -2\phi_{\max}(\sfA,m)\epsilon - x \ge -\frac{2\epsilon(1+x)}{1-2\epsilon}-x 
\Rightarrow \phi_{\min}(\sfA, m) \ge 1 - x -  \frac{2\epsilon(1+x)}{1-2\epsilon} = 1 - \frac{2\epsilon + x}{1-2\epsilon}.
\$
Taking $\epsilon =\frac{1}{64}$, we have $\phi_{\min}(\sfA, m) \ge 1 - 2x - \frac{1}{16}$
and $\phi_{\max}(\sfA, m) \le 1 + 2x +\frac{1}{16}$. 
In addition, for any $v,r \in \RR^d$ such that $\supp(v)\cap \supp(r) = \varnothing$ and $\|v\|_0+\|r\|_0 \le m$, we have 
\@ \label{eq:cross_term}
|v^\top \sfA r| = \bigg|\frac{(v+r)^\top \sfA(v+r) - (v-r)^\top \sfA (v-r)}{4} \bigg| \le x + \frac{1}{32}. 
\@

We now apply the above results with $m = 3s$.
Next, for an arbitrary vector $u$ satisfying $\|u[{S^c}]\|_1 \le 3 \|u[S]\|_1$ 
for some $S \subseteq [d]$ with $|S|\le s$, we decompose its 
coordinates into blocks of size $s$ with the first block $\sfI_1$ containing the indices in $S$, 
the second block $\sfI_2$ containing the indices of the $s$ largest coordinates in $S^c$, and so on. 
Let $K$ denote the number of blocks. We can check that 
\$ 
\frac{1}{T_1} \sum^{T_1}_{t=1} (a_t^\top u)^2 
& = \frac{1}{T_1} \sum^{T_1}_{t=1} \Big(\sum^K_{k=1} \sum_{j \in \sfI_k } u_j a_{t,j}\Big)^2\\
& \ge \frac{1}{T_1}\sum^{T_1}_{t=1}  \Big(\sum_{j \in \sfI_1 \cup \sfI_2 } u_j a_{t,j}\Big)^2
+ \frac{2}{T_1}\sum^{T_1}_{t=1} \Big(\sum^K_{k=3} \sum_{j \in \sfI_k } u_j a_{t,j}\Big) \Big(\sum_{j \in \sfI_1 \cup \sfI_2 } u_j a_{t,j}\Big).
\$
The first term on the right-hand side can be lower bounded as $\phi_{\min}(\sfA, 3s) \|u[\sfI_1 \cup \sfI_2]\|_2^2$.
As for the second term, we have
\$ 
\bigg|\frac{1}{T_1} \sum^{T_1}_{t=1}
\Big(\sum^K_{k=3} \sum_{j \in \sfI_k } u_j a_{t,j}\Big) \Big(\sum_{j \in \sfI_1 \cup \sfI_2 } u_j a_{t,j}\Big)\bigg|
& \le \sum^K_{k=3} \bigg| \frac{1}{T_1} \sum^{T_1}_{t=1}\Big(\sum_{j \in \sfI_k } u_j a_{t,j}\Big) \Big(\sum_{j \in \sfI_1 \cup \sfI_2 } u_j a_{t,j}\Big)\bigg|\\
& \le \Big(x+\frac{1}{32}\Big) \sum^K_{k=3} \|u[\sfI_k]\|_2 \|u[\sfI_1 \cup \sfI_2]\|_2,
\$
where the last step uses the result in Equation~\eqref{eq:cross_term}. By the definition of the blocks, we have for any $k\ge 3$ that 
\$ 
\sum_{k\ge 3}\|u[\sfI_k]\|_2 \le \sum_{k\ge 3}\frac{1}{\sqrt{s}}\|u[\sfI_{k-1}]\|_1 
\le \frac{1}{\sqrt{s}} \|u[\sfI_1^c]\|_1 \le \frac{3}{\sqrt{s}} \|u[\sfI_1]\|_1 \le 3 \|u[\sfI_1]\|_2. 
\$ 
Putting everything together, we have that with probability at least $1-2\exp(-(\frac{x^2}{64}\wedge \frac{x}{8})T_1 + 
3s\log(130d))$ 
that for any $u$ satisfying $\|u[S^c]\|_1 \le 3 \|u[S]\|_1$, for some $S \subseteq [d]$ with $|S|\le s$,
\$ 
\frac{1}{T_1} \sum^{T_1}_{t=1} (a_t^\top u)^2 \ge \phi_{\min}(\sfA, 3s) \|u[\tilde S]\|_2^2 - (6x + \frac{3}{16}) \|u[\tilde S]\|_2^2
\ge \Big(1 - 7x - \frac{1}{4}\Big) \|u[\tilde{S}]\|_2^2.
\$
Taking $x = 8 \sqrt{(3s\log(130d) + \log(2/\delta))/T_1}$ and using the assumption, 
we have that with probability at least $1-\delta$,
\$ 
\frac{1}{T_1} \sum^{T_1}_{t=1} (a_t^\top u)^2 \ge \frac{1}{2} \|u[\tilde{S}]\|_2^2,
\$
for any $u$ satisfying $\|u[S^c]\|_1 \le 3 \|u[S]\|_1$, for some $S \subseteq [d]$ with $|S|\le s$.

\section{Additional simulation results}\label{app:experiments}


There is an implementation detail to cap the computational cost of the baseline algorithm. At each round of optimization, if the optimizer cannot find a solution before a fixed time limit, the algorithm simply samples a random action vector from the previous actions. This ensures the algorithm does not spend an unbounded amount of time at any round.

We performed robustness checks for the following methods:
(i) the baseline,
(ii) no network information (NETC, Algorithm~\ref{alg:netc}),
(iii) partial network information (NSE, Algorithm~\ref{alg:col-support}), and
(iv) full network information (NSE-FS, Algorithm~\ref{alg:full-support}).

The setting for Figure~\ref{fig:sim_signal} is the same as in Section~\ref{sec5}, with signal strengths in $\{0.01, 0.05, 0.1, 0.15, 0.2, 0.5\}$. For each signal strength, the experiment is run 100 times on different random seeds. The treatment effects are generated as in \eqref{eq:signal_function} with different $\texttt{signal\_strength}$ values. As the signal strength increases, all methods incur lower regret in absolute terms, but the relative ordering shifts: at very small signal strengths (e.g., $0.01$), all four methods perform similarly because the underlying policy gap is small and distinguishing signals from noise is hard; at moderate signal strengths (around $0.1$, the setting used in the main text), the structure-aware algorithms (NSE and NSE-FS) achieve the lowest regret, while NETC sits between them and the baseline; at larger signal strengths (e.g., $0.2$ and $0.5$), NETC's explore-then-commit strategy benefits disproportionately because the post-exploration greedy action is much more likely to be correct, so NETC can match or outperform the structure-aware methods. 

The setting for Figure~\ref{fig:sim_sparsity} is also the same as in Section~\ref{sec5} except that the sparsity parameter $s$ takes values in $\{5, 10, 15, 20, 25, 50\}$ while $d$ is fixed at $100$. For each sparsity level, the experiment is run 100 times on different random seeds. When the network is very sparse ($s=5$), all algorithms perform similarly since there is little interference to learn. As $s$ increases, the baseline's regret grows most rapidly, consistent with its $\tilde{\mathcal{O}}(d\sqrt{dT})$ rate that does not benefit from sparsity, while the structure-aware algorithms continue to outperform it. At the densest setting ($s=50$, i.e., half the population), NSE-FS (full support knowledge) provides the most substantial improvement, highlighting the value of structural information when the network is dense. 

\begin{figure}[t]
\centering
\includegraphics[width=\textwidth]{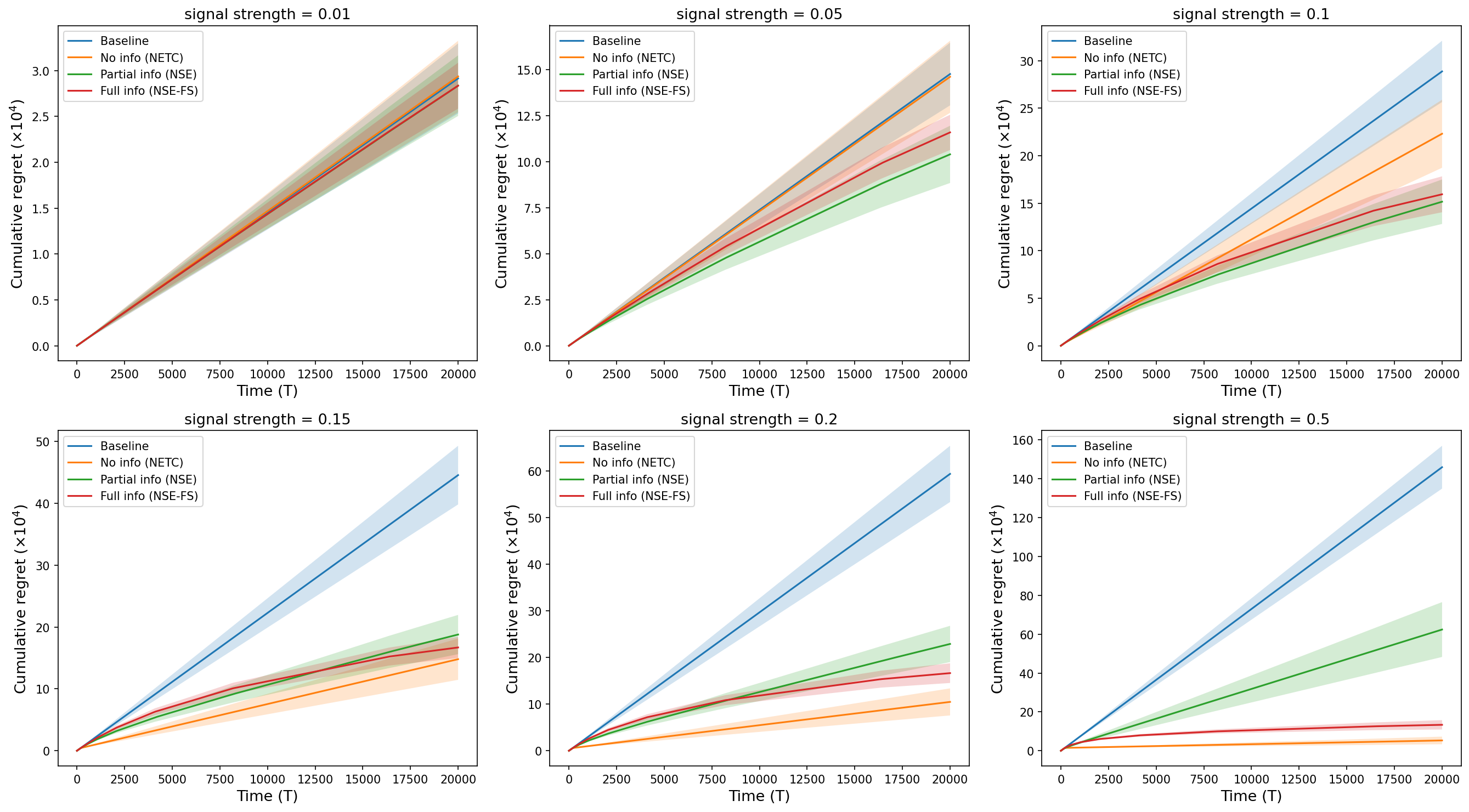}
\caption{Cumulative regret comparison for different signal strengths.}\label{fig:sim_signal}
\end{figure}

\begin{figure}[t]
\includegraphics[width=\textwidth]{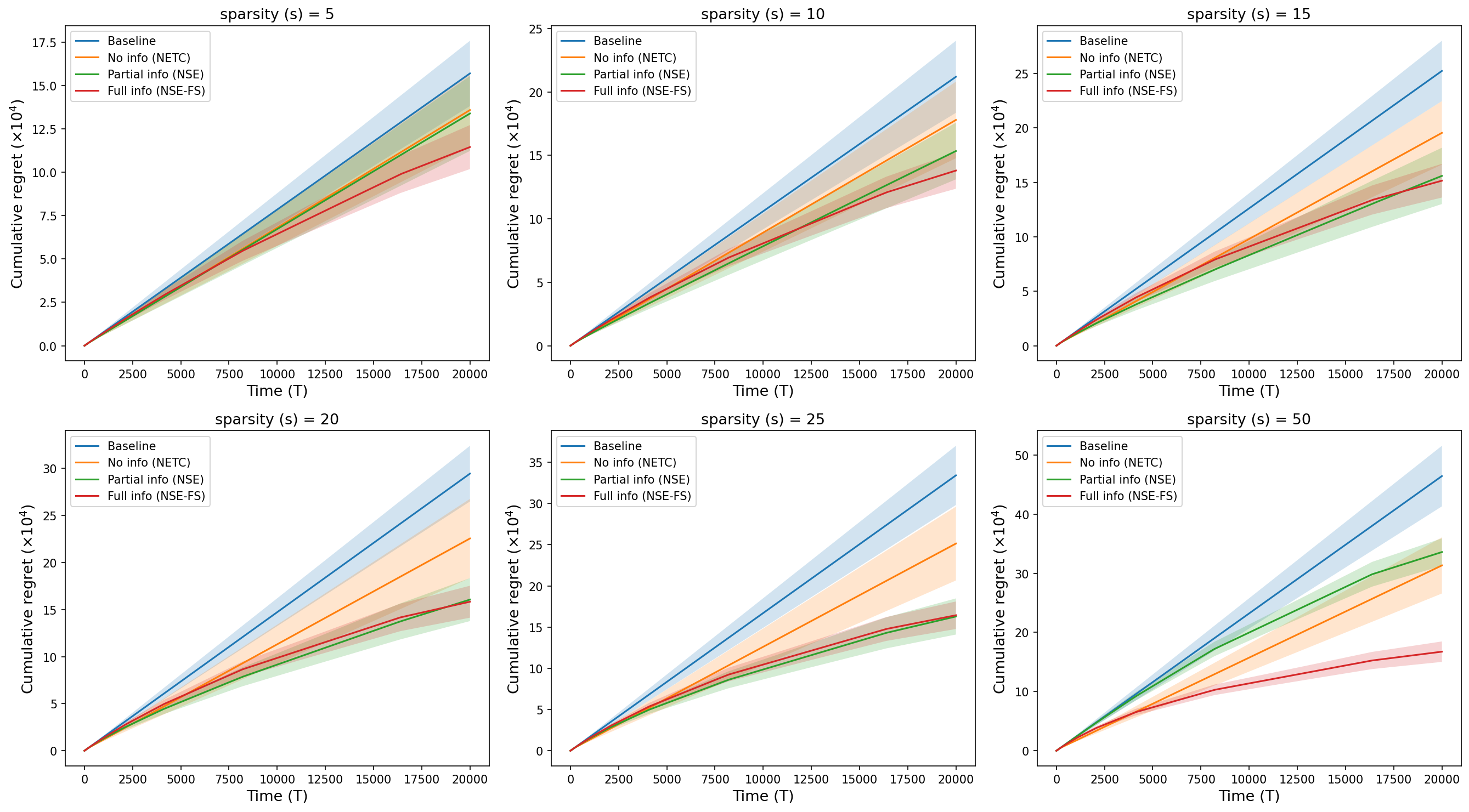}
\caption{Cumulative regret comparison for different sparsity levels.}\label{fig:sim_sparsity}
\end{figure}

\end{document}